\def\year{2022}\relax
\documentclass[letterpaper]{article} 
\usepackage{aaai22}  
\usepackage{times}  
\usepackage{helvet}  
\usepackage{courier}  
\usepackage[hyphens]{url}  
\usepackage{graphicx} 
\usepackage{natbib}  
\usepackage{caption} 
\DeclareCaptionStyle{ruled}{labelfont=normalfont,labelsep=colon,strut=off} 
\usepackage{algorithm}
\usepackage{algorithmic}

\usepackage{newfloat}
\usepackage{listings}
\usepackage{amsmath}
\usepackage{comment}
\usepackage{subfigure}
\usepackage{amsthm,amssymb}
\usepackage{textcomp,booktabs}
\usepackage{footmisc}
\usepackage{multirow}
\usepackage{bigstrut}
\usepackage{array}
\usepackage{amsfonts}
\usepackage{bm}
\usepackage{bbm}
\usepackage{threeparttable}
\usepackage[switch]{lineno}  %

\newcommand{\F}{\mathcal{F}}
\newcommand{\Ln}{\mathcal{L}_{nll}}
\newtheorem{definition}{Definition}

\newcommand{\myfrac}[2]{\left.#1\middle/#2\right.}

\floatstyle{ruled}
\newfloat{listing}{tb}{lst}{}
\floatname{listing}{Listing}

\title{Active Learning for Domain Adaptation: An Energy-Based Approach}
\author{
    Binhui Xie,\textsuperscript{\rm 1} Longhui Yuan,\textsuperscript{\rm 1} Shuang Li,\textsuperscript{\rm 1, \thanks{Corresponding author.}} Chi Harold Liu,\textsuperscript{\rm 1} Xinjing Cheng,\textsuperscript{\rm 2, \rm 3} Guoren Wang\textsuperscript{\rm 1}
}
\affiliations{
    \textsuperscript{\rm 1}School of Computer Science and Technology, Beijing Institute of Technology, Beijing, China  \\
    \textsuperscript{\rm 2}School of Software, BNRist, Tsinghua University, Beijing, China     \\ 
    \textsuperscript{\rm 3}Inceptio Technology, Shanghai, China 	\\

    \{binhuixie, longhuiyuan, shuangli, chiliu, wanggr\}@bit.edu.cn, cnorbot@gmail.com
}

\usepackage{bibentry}

\begin{document}

\maketitle

\begin{abstract}
    Unsupervised domain adaptation has recently emerged as an effective paradigm for generalizing deep neural networks to new target domains. However, there is still enormous potential to be tapped to reach the fully supervised performance. In this paper, we present a novel active learning strategy to assist knowledge transfer in the target domain, dubbed active domain adaptation. We start from an observation that energy-based models exhibit \textit{free energy biases} when training (source) and test (target) data come from different distributions. Inspired by this inherent mechanism, we empirically reveal that a simple yet efficient energy-based sampling strategy sheds light on selecting the most valuable target samples than existing approaches requiring particular architectures or computation of the distances. Our algorithm, \textit{Energy-based Active Domain Adaptation (EADA)}, queries groups of target data that incorporate both domain characteristic and instance uncertainty into every selection round. Meanwhile, by aligning the free energy of target data compact around the source domain via a regularization term, domain gap can be implicitly diminished. Through extensive experiments, we show that EADA surpasses state-of-the-art methods on well-known challenging benchmarks with substantial improvements, making it a useful option in the open world. Code is available at \url{https://github.com/BIT-DA/EADA}.
\end{abstract}

\section{Introduction}
In recent years, we have witnessed great strides in diverse machine learning problems with the success of deep neural networks~\cite{alexnet}. At the moment, however, these leaps in performance come only when labeled data is abundant. This limits their usage in many practical applications, such as autonomous driving with massive unlabeled data~\cite{YogamaniWRNMVPO19} and medical diagnosis with high labeling cost~\cite{ronneberger2015u-net}. Moreover, even labeling all available data is not an excellent solution, as it's impossible to fully capture the way the world looks in a single dataset, let alone the fact that the test data rarely matches the data seen during training. Recognizing these challenges, domain adaptation (DA) has been studied extensively, which transfers models trained on a labeled source domain to an unlabeled target domain~\cite{survey,Tzeng_2015_Simultaneous,ADDA,DANN,CDAN,BousmalisSDEK17,DRCN,TSA,MCD}. The performance of DA, in spite of great success, often falls far behind that of supervised learning. In practice, it is feasible to obtain extra annotations for target data, but to save cost, it is better to select the most informative subset via active learning~\cite{prince2004does,hanneke2014theory,Bickel_2009_JMLR}.

While previous active learning studies drastically lower human annotation costs, they are impractical when test data are collected from out-of-distribution. How can we design an efficient and practical sampling strategy for domain adaptation? For one thing, it is essential to determine which target samples will, once labeled, boost the accuracy and generalization considerably. For another, it remains the boundary to explore how to effectively utilize limited labeled data from the target domain to perform adaptation. Aware of this need, researchers have developed an array of active domain adaptation (Active DA) methods~\cite{Chattopadhyay_2013_ICML,rai2010domain,AADA_WACV,Fu_2021_CVPR,CULE_2021_ICCV,ChanN07}. Prior works mainly focus on assessing how private each target data is according to the output of a domain discriminator or calculating its distance to the cluster centroids. However, these additional procedures either select target samples that are originally well aligned with the source domain or increase the computational overhead, which limits their capability. Therefore, a simple yet efficient solution is urgently desired.

\begin{figure*}[!htbp]
  \centering  
    \subfigure[Source Only]{
      \label{Fig_motivation_source}
      \includegraphics[width=0.23\textwidth]{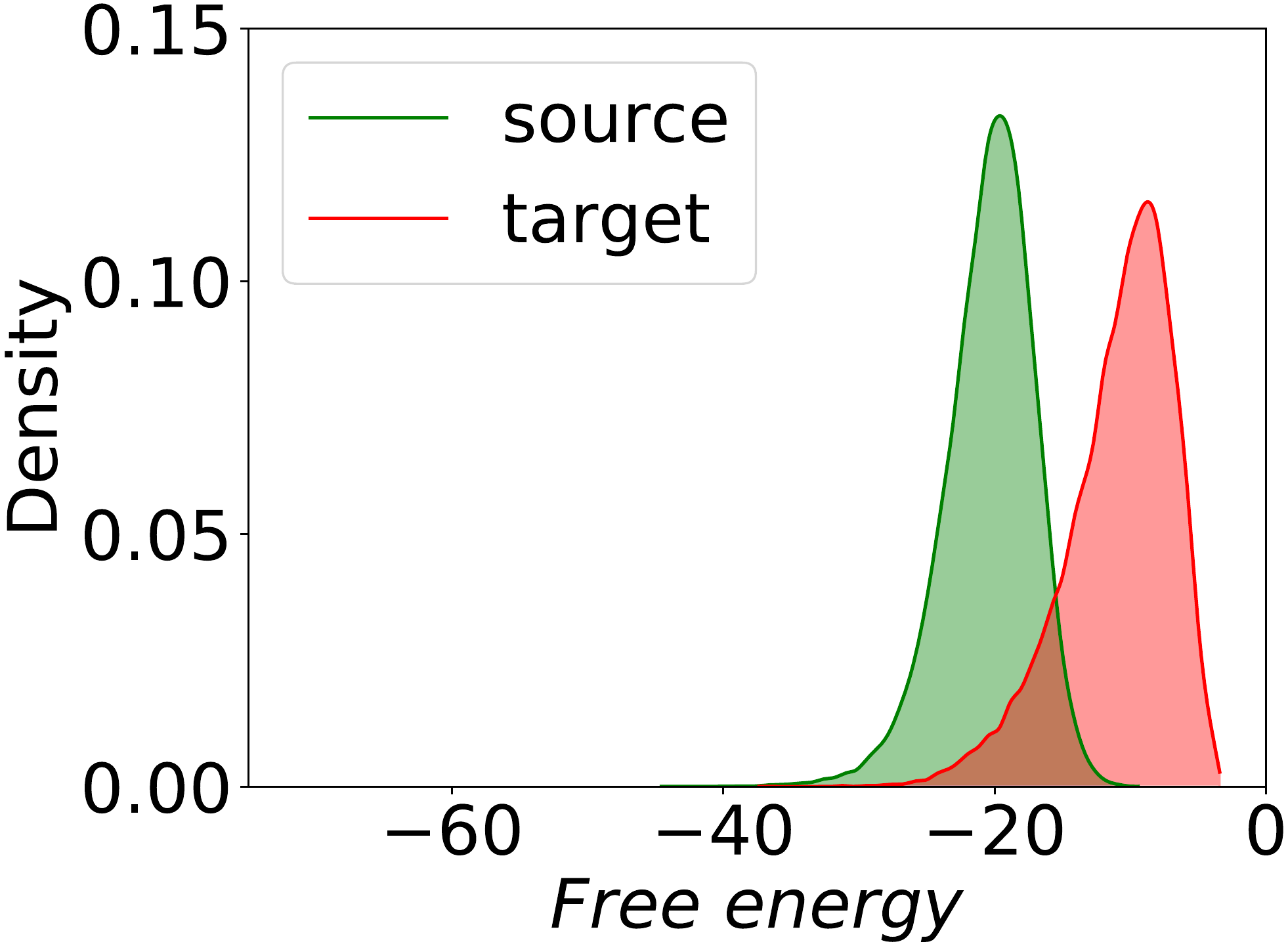}
    }
    \subfigure[Fully Supervised]{
      \label{Fig_motivation_full}
      \includegraphics[width=0.23\textwidth]{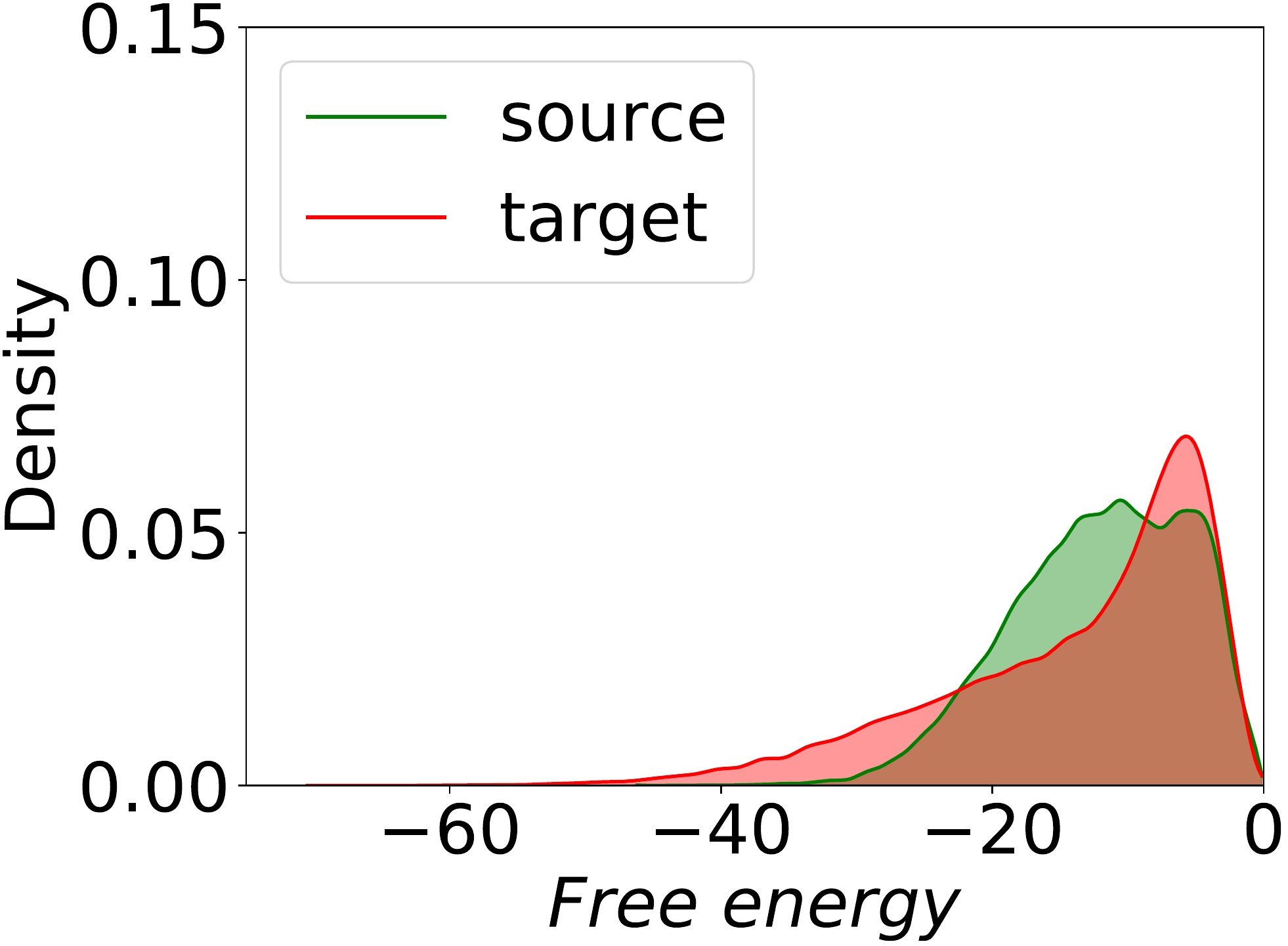}
    }
    \subfigure[Random]{
      \label{Fig_motivation_random}
      \includegraphics[width=0.23\textwidth]{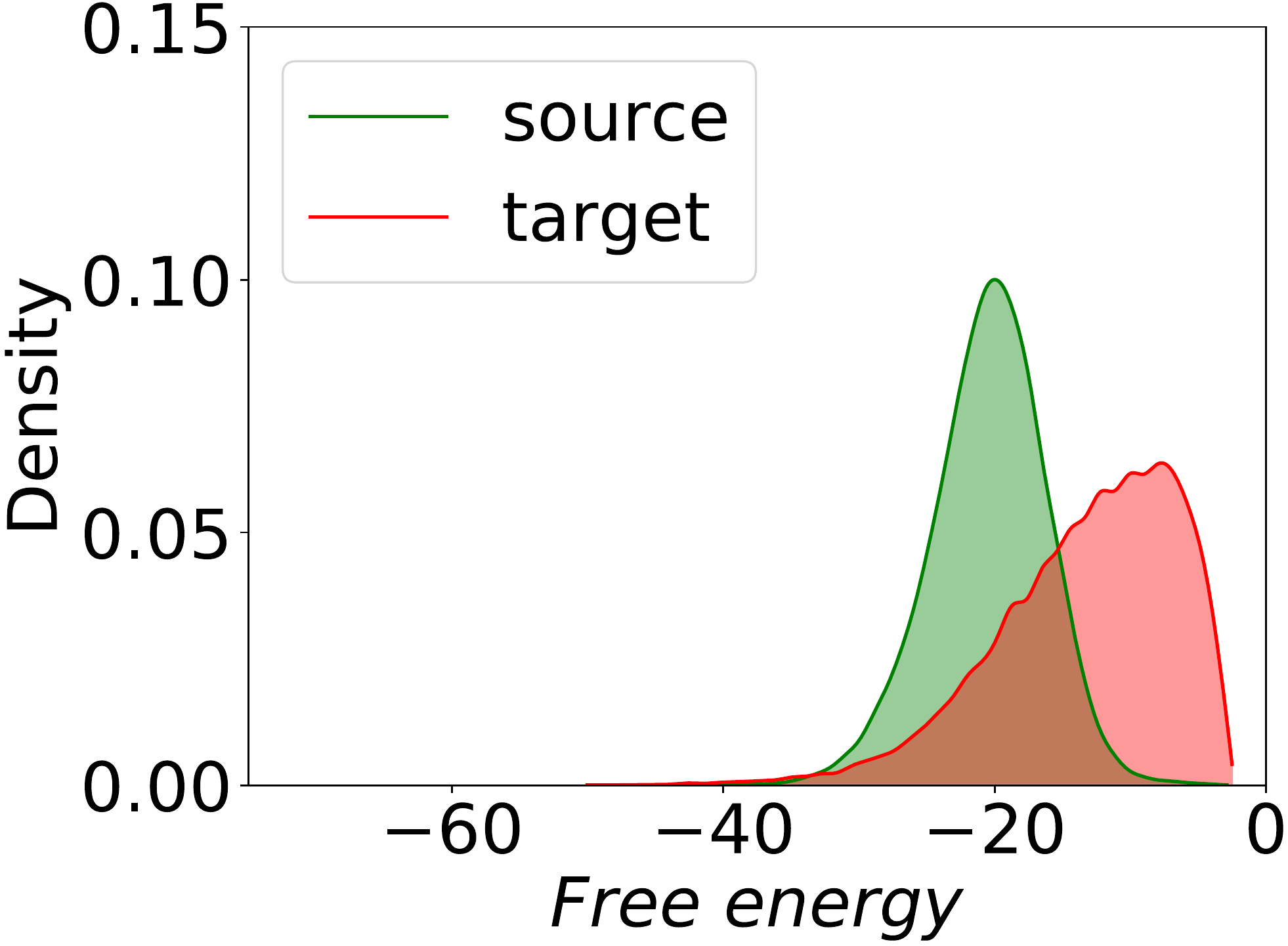}
    }
    \subfigure[EADA (Ours)]{
      \label{Fig_motivation_ours}
      \includegraphics[width=0.23\textwidth]{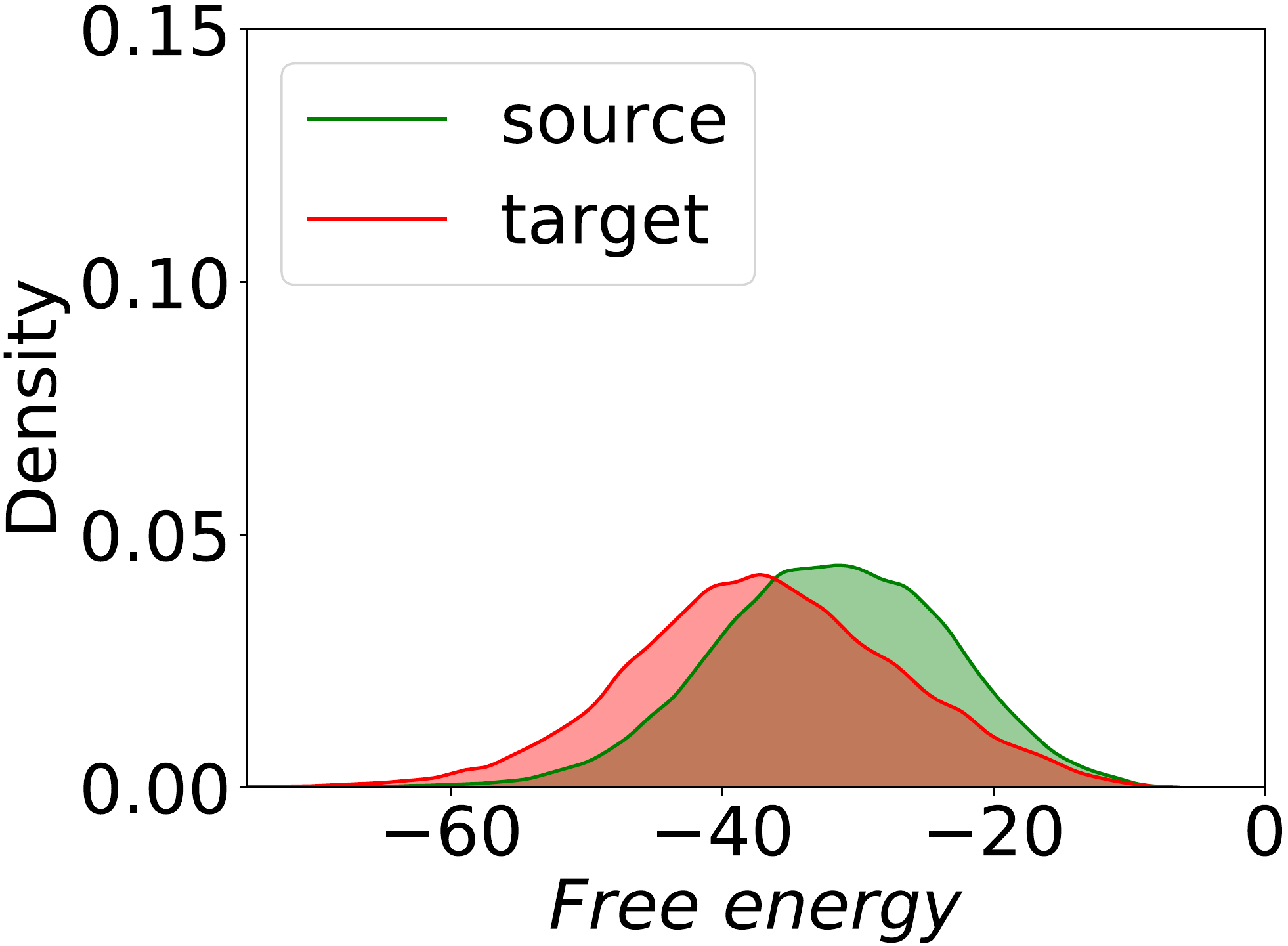}
    }
  \caption{(a \& b) Free energy distribution biases between source and target domains on VisDA-2017 from ``Source Only". We then contrast the distributions from (c) a native baseline of random selection and (d) our energy-based selection strategy. EADA exhibits better-aligned distribution than ``Random" and is similar to ``Full Supervised'' i.e., all source and target data are labeled.}
  \label{Fig_motivation}
\end{figure*}

In this paper, we advocate the use of energy-based models (EBMs)~\cite{lecun2006tutorial} to help realize the potential of active learning under domain shift. For any given $x$ (e.g., an image), an EBM approach gives the lowest energy to the correct answer $y$ (e.g., a label). \citet{Energy1} and \citet{Energy2} have demonstrated that energy-based training improves calibration and better distinguishes in- and out-of-distribution samples than the standard discriminative classifier. At this point, we begin with investigating the distributions of free energy on source and target domains using diverse methods and make several observations from Fig.~\ref{Fig_motivation}. First, a model trained only on labeled source data will cause the free energy distribution of the supervised source data to be lower than that of the unlabeled target data, that is, \textit{free energy biases} between the two domains (Fig.~\ref{Fig_motivation_source}). Then, an interesting finding is that these two distributions tend to be consistent using ``Full Supervised'' (Fig.~\ref{Fig_motivation_full}). Next, the biases eliminate slightly when a few unlabeled target data are randomly annotated in the training process (Fig.~\ref{Fig_motivation_random}). Lastly, using our algorithm to identify limited target instances for labeling, surprisingly, it well matches two distributions as same as the situation of ``Full Supervised" (Fig.~\ref{Fig_motivation_ours}). We conjecture that there exist redundant or trivial data in the target domain itself that do little to help the learning objective. 

Intuitively, we provide both mathematical insights and empirical evidence that an energy-based active learning scheme is desirable for active domain adaptation. A central theme of this work is that we design an approach, Energy-based Active Domain Adaptation (EADA), which adequately ensures samples that are representative of the entire target domain to be selected by considering both domain characteristic and instance uncertainty. More precisely, as mentioned above, the free energies of most labeled source data are lower than that of unlabeled target data. Thus, we can treat the intrinsic free energy of an unlabeled target sample as a surrogate metric to reflect the domain characteristic. Naturally, the target samples with higher free energy are more dissimilar to source data, and thus be typical for target distribution. In addition, we assess the value of minimum energy versus second-minimum energy (MvSM) for each unlabeled target data to quantify its uncertainty under the current model. To this end, given the labeling budget in each round, we first maintain a candidate set from unlabeled target data with higher free energies and then select samples with significant MvSM values from candidates. Furthermore, free energy can also serve as a regularization signal in the form of an alignment loss to implicitly diminish domain shift, which is complementary to our active strategy. 

In summary, our work makes the following contributions:
\begin{itemize}
    \item We provide a new perspective to select a highly informative subset of unlabeled target data under domain shift via exploiting \textit{free energy biases} between the two domains.
    \item We complement empirical results with theoretical investigations in the method section and establish an intuitive sufficient condition when it would  help. 
    \item Though simple, EADA attains excellent results with quite limited labeling expenses. Extensive experiments and in-depth analysis demonstrate its effectiveness.
\end{itemize}

\section{Related Work}\label{sec:related}
\subsubsection{Active learning (AL)} has been studied for decades in both theory and practice~\cite{settles2009active,active2011Sanjoy,bachman2017learning,UncertaintyBased1}. A case in point is to search informative data for labeling in order to learn a satisfactory model at a low annotation cost. Most popular algorithms formulate and solve it by uncertainty sampling. They select samples about which the current model is uncertain~\cite{UncertaintyBased2,BvSB_2009_CVPR,entropy_2014_IJCNN}. Another line of work turns to representative sampling~\cite{CoreSet_2019_ICLR,RepresentativeBased2,RepresentativeBased3}, which picks a set of typical samples via clustering or core-set selection. 

Recently, several studies have leveraged a hybrid of the above active sampling objectives to achieve promising results, such as \citet{BADGE_2020_ICLR}. However, these conventional AL methods cannot deal with the domain shift issues for domain adaptation, whereas our method aims to overcome this challenge by leveraging a simple energy-based strategy. 

\subsubsection{Domain adaptation (DA)} studies the task of transferring knowledge gained from a labeled source domain to a target domain where annotations are scarce~\cite{DANN,JAN,DAN-PAMI,AFN2019Xu,DICD,DCAN,BCDM,Hoffman_CyCADA,zou2021EADA,GFK2012Gong}. A series of works minimizes the domain gap at the uppermost layer of deep networks using maximum mean discrepancy~\cite{MMD} or adversarial training~\cite{GAN}. Recently, some methods allow a few target data labeled, e.g., semi-supervised DA~\cite{saito2019MME} and few-shot DA~\cite{FSDA}. Though impressive, they randomly select a few data to annotate, neglecting which target samples should be labeled given a fixed labeling budget. Consequently, some selected samples are originally well predicted by the current model. In contrast, our work differentiates itself by allowing the model to acquire labels for valuable target samples via an oracle. As such, it would have the best potential performance gain compared with randomly picking labels. 

\subsubsection{Active domain adaptation (Active DA).} The seminal work~\cite{rai2010domain} has demonstrated the synergy between AL and DA, which facilitates AL in a domain of interest with the aid of the knowledge from a related domain. Recently, \citet{AADA_WACV} and \citet{Fu_2021_CVPR} incorporate Active DA with advanced tools, such as adversarial training, both of which identify domainness via a learned domain discriminator. However, it may give identically high scores to most target data, thus not adequately ensuring that selected samples are representative of the entire target distribution. A parallel line of work instead proposes to select active samples via clustering. For example, \citet{CULE_2021_ICCV} cluster deep embeddings of target data weighted by the uncertainty and select nearest neighbors to the inferred cluster centroids for labeling. However, clustering-based strategies have some drawbacks in nature. First, they encounter a computational burden and could hardly be applied on large data sets. Second, the clustering is sensitive to noise and easy to collapse.  

Originating from energy-based models, our method adapts the concept of energy to identify limited target samples that are most unique to the target distribution and meanwhile complementary to labeled source data. It yields a new sampling protocol that accounts for domain characteristic and instance uncertainty together. Also, it has no extra parameters that need to be optimized and learning is efficient.

\section{Method}\label{sec:method}
\begin{figure*}
    \centering
    \includegraphics[width=0.96\textwidth]{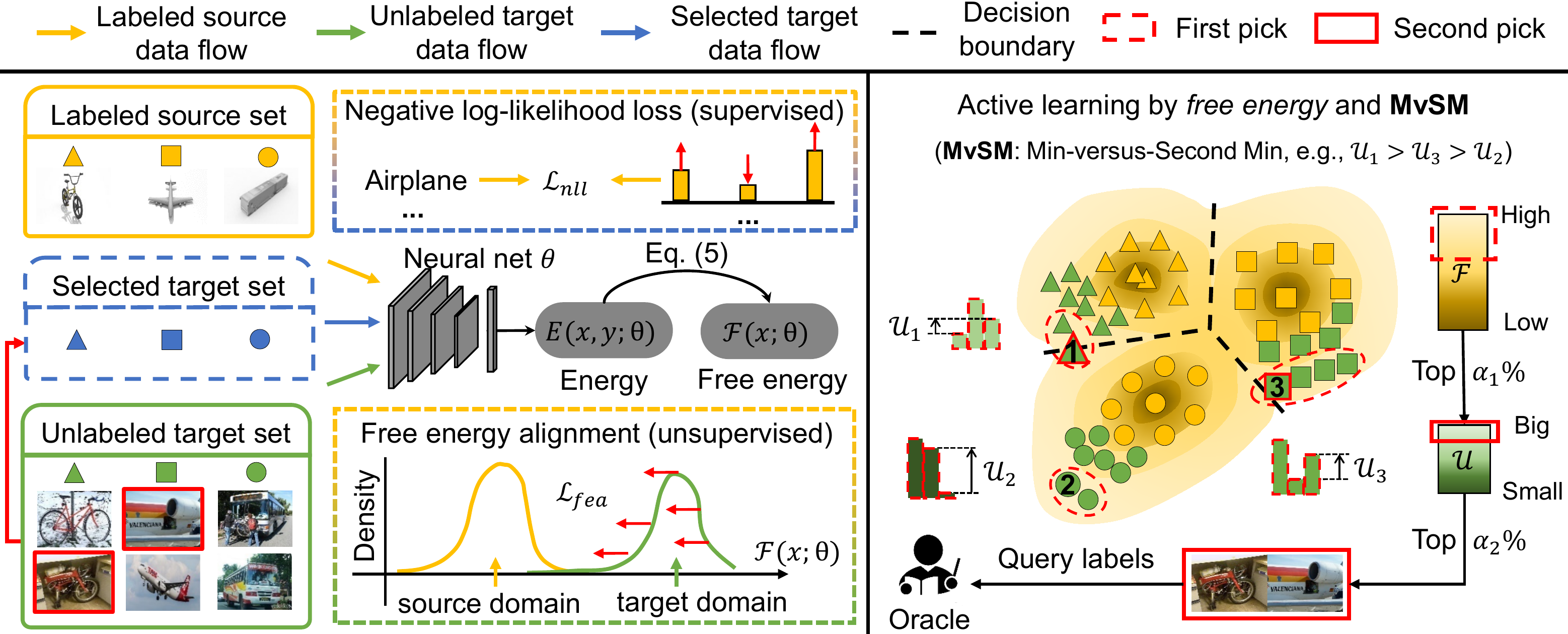}
    \caption{Overview of the EADA. \textit{Left (training process)}: we utilize the standard negative log-likelihood loss in conjunction with the proposed free energy alignment loss to train the network; \textit{Right (selection process)}: to iteratively build a labeled target set, 1\% of target samples are quired to annotate in each selection round. We first select a set of $\alpha_1\%$ candidates with the highest free energy (domain characteristic). Then we sample $\alpha_2\%$ points from candidates with biggest MvSM (instance uncertainty).}
    \label{Fig_framework}
\end{figure*}

In Active DA, we have access to a labeled source domain $\mathcal{S} = \{(x_s,y_s)\}$ and an unlabeled target domain $\mathcal{T} = \{x_t\}$ from different distributions. Following the standard Active DA setting~\cite{Fu_2021_CVPR,CULE_2021_ICCV}, $B$ active samples that are much smaller than the amount of $\mathcal{T}$ are selected for annotating. Thus, the entire target domain consists of a labeled pool $\mathcal{T}_{l}$ and an unlabeled pool $\mathcal{T}_{u}$, i.e., $\mathcal{T}=\mathcal{T}_{l} \cup \mathcal{T}_{u}$. The goal is to learn a neural network with parameter $\theta$ that brings good generalization on the target domain. In this work, we introduce an energy-based strategy to select the most valuable target data to assist the adaptation.

\subsection{Energy-based Models Revisit}
The essence of machine learning is to encode dependencies between variables. Let us consider an energy-based model (EBM) with two sets of variables $x$ (a high-dimensional variable) and $y$ (a discrete variable). Training this model consists in finding an energy function i.e., $E(x,y)$ that gives the lowest energy to correct answer and higher energy to all other (incorrect) answers\footnote{See ~\cite{lecun2006tutorial} for a comprehensive tutorial.}. Precisely, the model must produce the value $y^*$ for which $E(x,y)$ is the smallest:
\begin{small}
    \begin{equation}
        \begin{aligned}
            y^* = {\arg\min}_{y\in \mathcal{Y}} E(x, y) \,.
        \end{aligned}
    \end{equation}
\end{small}%
Generally, the size of set $\mathcal{Y}$ is small for classification, hence the inference procedure can simply compute $E(x, y)$ for all possible values of $y\in \mathcal{Y}$ and pick the smallest.

With the energy function, the joint probability of input $x$ and label $y$ can be estimated through the Gibbs distribution:
\begin{small}
    \begin{equation}
        \begin{aligned}
        p(x,y)=\myfrac{\exp\left(-E(x,y) \right)}{Z} \,,
    \end{aligned}
    \end{equation}
\end{small}%
where $Z=\sum_{x\in\mathcal{X}} \sum_{y\in\mathcal{Y}}\exp\left(-E(x,y)\right)$ is called the partition function that marginalizes over $x$ and $y$. It should be noted that the above transformation of energy into probability is only possible if $Z$ converges.
By marginalizing out $y$, we obtain the probability density for $x$ as well,
\begin{small}
\begin{equation}
        \begin{aligned}
    p(x) = \sum\nolimits_{y\in \mathcal{Y}}p(x,y) = \myfrac{\sum\nolimits_{y\in\mathcal{Y}}\exp\left(-E(x,y)\right)}{Z} .
\end{aligned}
\label{eq:px_energy}
    \end{equation}
\end{small}%

Intuitively, in Active DA, to select the most representative target samples, one can directly estimate the probability of occurrence for each target sample from Eq.~\eqref{eq:px_energy} and then those samples with lower probabilities should be selected. 

Unfortunately, one cannot compute or even reliably estimate $Z$. Therefore, we turns to \textit{free energy} i.e., $\mathcal{F}(x)$, a function hidden in EBMs that serves as the ``rationality" of the occurrence of the variable $x$. Mathematically, the probability density for $x$ can also be expressed as
\begin{small}
\begin{equation}
    \begin{aligned}
    p(x)=\frac{\exp\left(-\mathcal{F}(x)\right)}{\sum_{x\in\mathcal{X}}\exp(-\mathcal{F}(x))}\,.
    \label{eq:px_free_energy}
    \end{aligned}
\end{equation}
\end{small}%
This formulation indicates that $\mathcal{F}(x)$ could be substituted for $p(x)$ to select the target samples that have lower probabilities. By connecting Eq.~\eqref{eq:px_energy} and Eq.~\eqref{eq:px_free_energy}, we have
\begin{small}
\begin{equation}
        \begin{aligned}
    \mathcal{F}(x)=-\log\sum_{y\in\mathcal{Y}}\exp\left(-E(x,y)\right)\,.
    \label{eq:active_domain}
\end{aligned}
    \end{equation}
\end{small}%

\subsection{Energy-based Active Domain Adaptation}
We take advantage of a new perspective of EBMs to gain the benefits for active domain adaptation, where \textit{free energy biases} between source and target data allow effective selection and adaptation. In the following, we first describe how to train an EBM with several loss functions. Then we show how to query and annotate the most informative unlabeled target data via an energy-based sampling strategy. At last, we provide an intuitive sufficient condition when it helps. 

\subsubsection{Training process} 
Given a set of labeled source samples $\mathcal{S} = \{(x_s,y_s)\}$, we want to train a well-behaved EBM that gives the lowest energy to the correct answer and higher energy to all other (incorrect) answers. To this end, we utilize a commonly used loss in EBMs, i.e., the negative log-likelihood loss that comes from probabilistic modeling to train a model for classification, and it can be formulated as
\begin{small}
    \begin{equation}
        \begin{aligned}
        \mathcal{L}_{nll}(x,y;\theta) = E(x,y;\theta) + \frac{1}{\tau} \log\sum_{c\in\mathcal{Y}} \exp\left(-\tau E(x,c;\theta)\right),
    \end{aligned}
    \end{equation}
\end{small}%
where $\tau$ ($\tau > 0$) is the reverse temperature and a low value corresponds to smooth partition of energy over the space $\mathcal{Y}$. For simplicity, we fix $\tau$=1, and then we have
\begin{small}
    \begin{equation}
        \begin{aligned}
        \mathcal{L}_{nll}(x,y;\theta) = E(x,y;\theta) - \mathcal{F}(x;\theta) \,.
        \label{eq:nll_energy_mine_free_energy}
    \end{aligned}
    \end{equation}
\end{small}%
The second term in Eq.~\eqref{eq:nll_energy_mine_free_energy} will cause the energies of all answers to be pulled up. The energy of the correct answer is also pulled up, but not as hard as it is pushed down by the first term. An analysis of gradient is presented in Appendix\footnote{Appendix can be found at \url{https://arxiv.org/abs/2112.01406}.}.

However, we observe that the values of free energy on target samples are considerably higher than those on source ones, called \textit{free energy biases}. Naturally, one can treat it as a surrogate to reflect the domain divergence. By designing a simple regularization term, these biases can be reduced, which to some extent aligns the distribution across domains. And the free energy alignment loss $\mathcal{L}_{fea}$ is defined as:
\begin{small}
    \begin{equation}
        \begin{aligned}
        \mathcal{L}_{fea}(x;\theta) = \max \left(0, \mathcal{F}(x;\theta) - \Delta\right),
    \end{aligned}
    \end{equation}
\end{small}%
where $\Delta = \mathbb{E}_{x\sim \mathcal{S}} \mathcal{F}(x;\theta)$ is the average value of the free energy over source data. During training, $\Delta$ is estimated via exponential moving average: $ \Delta_t = \lambda \Delta_{t-1} + (1 - \lambda) \Delta'_t$, where $\Delta_t$ is the estimation of average value in all $t$ mini-batches and $\Delta'_t$ is the average value in $t^{th}$ mini-batch and $\lambda$ is a weight sampled from the uniform distribution, $\lambda\sim U(0,1)$. Additionally, we experimentally found that such way is comparable with calculating average value over the whole source domain data while improving the efficiency.

Overall, the full learning objective is given by:
\begin{small}
    \begin{equation}
        \begin{aligned}
        \min_{\theta} \mathbb{E}_{(x,y)\sim\mathcal{S}\cup\mathcal{T}_{l}}\mathcal{L}_{nll}(x,y;\theta) + \gamma \mathbb{E}_{x\sim \mathcal{T}_u}\mathcal{L}_{fea}(x;\theta),
        \label{eq:overall}
    \end{aligned}
    \end{equation}
\end{small}%
where $\gamma$ is a loss weight hyperparameter. 

\subsubsection{Selection process} The goal in Active DA is to identify more valuable target samples that, once labeled and used for training, improve the model's accuracy and generalization performance significantly. In practice, we suggest a two-step sampling strategy to adequately ensure such samples by incorporating domain characteristic and instance uncertainty. To be clear, we summarize the training and selection processes based on the above discussion as Algorithm~\ref{alg:algorithm}.

\begin{algorithm}[t]
    \begin{algorithmic}[1]
    \caption{\textbf{EADA algorithm}}
    \label{alg:algorithm}
    \STATE \textbf{Input:} Labeled source data $\mathcal{S}$, unlabeled target data $\mathcal{T}_u$ and labeled target set $\mathcal{T}_l = \emptyset$, maximum epoch $M$, selection rounds $R$, selection ratios $\alpha_1\,,\alpha_2$ \\
    \FOR{$m=1$ to $M$}
    \STATE Update model $\theta_m$ via Eq.~\eqref{eq:overall} 
    \IF{$m$ in $R$} 
    \STATE $\forall x \in \mathcal{T}_u$, compute free energy $\mathcal{F}(x)$ (Eq.~\eqref{eq:active_domain}) to serve as measure of domain characteristic
    \STATE $\mathcal{T}_l^r \leftarrow$ select $\alpha_1\%$ of $\mathcal{F}$ with the highest values
    \STATE $\forall x \in \mathcal{T}_l^r$, compute MvSM~~$\mathcal{U}(x)$ (Eq.~\eqref{eq:active_uncertainty}) to serve as measure of instance uncertainty
    \STATE $\mathcal{T}_l^r \leftarrow$ select $\alpha_2\%$ of $\mathcal{U}$ with the highest values as active samples for annotating, getting $\mathcal{T}_l=\mathcal{T}_l \cup \mathcal{T}_l^r$
    \ENDIF
    \ENDFOR
    \STATE \textbf{Output:} Final model parameters $\theta_M$
    \end{algorithmic}
\end{algorithm}

{\bf Step one}: we observe that biases of free energy distribution between source and target domains exhibit. Thus, we can utilize this intrinsic free energy of an unlabeled target sample as a surrogate metric to reflect the domain characteristic. Certainly, the target samples with higher free energy are unique to the target distribution and meanwhile complementary to the labeled source data.

{\bf Step two}: to measure instance uncertainty, existing methods rely primarily on the entropy score~\cite{AADA_WACV,CULE_2021_ICCV}. In contrast, we consider the difference between the energy values of the two answers with the lowest estimated energy value as a measure of uncertainty. Since it is a comparison of the minimum answer and the second minimum answer, we refer to it as the Min-versus-Second-Min (MvSM) strategy and it can be formulated as 
\begin{small}
    \begin{equation}
        \begin{aligned}
        \mathcal{U}(x) = E(x, y^*;\theta) - E(x, y';\theta)\,,
    \end{aligned}
    \label{eq:active_uncertainty}
    \end{equation}
\end{small}%
where $y^* = \arg\min_{y\in\mathcal{Y}} E(x, y;\theta)$ is the lowest energy output and $y' = \arg\min_{y\in\mathcal{Y}\backslash \{y^*\}} E(x,y;\theta)$ is the second-lowest energy output. Such a measure is a more direct way of estimating confusion about class membership from a classification standpoint. Using the MvSM measure, the instances around the decision boundaries in Fig.~\ref{Fig_framework} during the selection procedure will be selected to query an oracle.

\subsection{Theoretical Analysis}
\label{sec:theoretical}
This section contains our preliminary study of why \textit{free energy biases} exhibit between two different domains. For an energy-based model, we prove that positive gradient inner product between the negative log-likelihood loss function and \textit{free energy} leads to a lower value of \textit{free energy} on labeled source samples during the training process. Limited by space, all the proofs are left for the Appendix.

Before stating our main theoretical result, we first illustrate the general intuition with a toy problem. Considering a simple energy-based model on the classification task, where the network is a one layer linear network parameterized by $\mathbf{W} 
= \begin{pmatrix}
    \omega_1 & \cdots & \omega_C
\end{pmatrix}^{\top} \in \mathbb{R}^{C \times N}$, $x\in\mathbb{R}^N$ denotes a source sample, $y\in\{1,...,C\}$ denotes the label, we have 
\begin{equation}
  \begin{aligned}
    E(x,j;\mathbf{W}) &= \omega_{j}^{\top} x,~j=1,...,C ,\\
    \mathcal{F}(x;\mathbf{W}) &= -\log\sum_{c = 1}^{C} \exp(-\omega_{c}^{\top} x) ,\\
    \mathcal{L}_{nll}(x, y;\mathbf{W}) &= E(x, y;\mathbf{W})-\mathcal{F}(x;\mathbf{W}) .
  \end{aligned}
\end{equation}
Now we update the the weight matrix $\mathbf{W}$ by one step of gradient descent on $\mathcal{L}_{nll}$ as follows:
\begin{small}
      \begin{equation}
        \mathbf{W}^{\prime} = \mathbf{W} - \eta \nabla \mathcal{L}_{nll}(x, y; \mathbf{W}) ,
    \end{equation}
\end{small}%
where $\eta$ is the learning rate and $\mathbf{W}^{\prime }$ is the updated matrix. 

Then we have two lemmas to show that the inner product between the gradients of negative log-likelihood loss function and free energy is positive, and the value of the free energy of a labeled source sample is descending with a step of gradient descent on negative log-likelihood loss function.
\newtheorem{lemma}{Lemma}
\begin{lemma}\label{lemma_toy_gradient}
    Assume that a toy model correctly predicts a labeled source sample $(x, y)$, we have 
      \begin{equation}
        \left\langle \nabla \mathcal{L}_{nll}(x, y;\mathbf{W}), \nabla \mathcal{F}(x;\mathbf{W})\right\rangle  > 0 , 
    \end{equation}
    where $\left\langle \cdot, \cdot \right\rangle $ denotes the inner product of gradients.
\end{lemma}
\begin{lemma}\label{lemma_toy_step}
    Assume that a toy model correctly predicts a labeled source sample $(x, y)$ with learning rate $\eta > 0$
    we have 
      \begin{equation}
        \mathcal{F}(x;\mathbf{W}) > \mathcal{F}(x;\mathbf{W}^{\prime}).
    \end{equation}
\end{lemma}%
\noindent To summarize, if the positive gradient inner product between the negative log-likelihood loss function and \textit{free energy}, free energy biases are exhibited. Our main theoretical results extend this to general deep neural networks.

\newtheorem{theorem}{Theorem}
\begin{theorem}\label{gradient_theorem}
    Let $\mathcal{L}_{nll}(x, y; \theta)$ denote the negative log-likelihood loss on source domain $(x,y)$ with parameters of deep network $\theta$ and $\mathcal{F}(x;\theta)$ denote the free energy of $x$. Assume that $\forall(x, y)$, $\mathcal{L}_{nll}(x, y; \theta)$ is differentiable, $\beta$-smooth in $\theta$ and $\forall \theta\,,\left\lVert \nabla \mathcal{L}_{nll}(x,y,\theta) \right\rVert < G, \left\lVert \nabla \mathcal{F}(x;\theta) \right\rVert < G $. With learning rate $\eta \in (0, \frac{2\varepsilon}{\beta G^2})$, and for every $(x, y)$ such that 
    \begin{equation}
        \left\langle \nabla\mathcal{L}_{nll}(x, y;\mathbf{\theta}), \nabla\mathcal{F}(x;\mathbf{\theta})\right\rangle  > \varepsilon ,
    \end{equation}
    where $\varepsilon > 0$, we have 
    \begin{equation}
        \mathcal{F}(x;\mathbf{\theta}) > \mathcal{F}(x;\mathbf{\theta}^{\prime}),
    \end{equation}
    where $\mathbf{\theta}^{\prime } = \mathbf{\theta} - \eta \nabla \mathcal{L}_{nll}(x, y; \mathbf{\theta})$ i.e., supervised training with one step of gradient descent, and $\left\langle \cdot, \cdot\right\rangle $ denotes the inner product of gradients.
\end{theorem}

The proof uses standard techniques in optimization~\cite{convex_optimization}. Theorem~\ref{gradient_theorem} reveals gradient correlation as a determining factor of the success of our algorithm. 

\section{Experiments}\label{sec:Experiment}
We evaluate EADA against prior arts on various scenarios including a toy example, three popular image classification datasets: VisDA-2017~\cite{visda2017}, Office-Home~\cite{Office-Home} and Office-31~\cite{Office31}, as well as a challenging semantic segmentation task, i.e., GTAV~\cite{GTA5} to Cityscapes~\cite{cityscapes}. All methods are implemented based on PyTorch, employing ResNet~\cite{resnet} models pre-trained on ImageNet~\cite{imagenet}. We follow the standard protocols as~\cite{AADA_WACV,Fu_2021_CVPR}. Meanwhile, the various compared active learning, active domain adaptation and domain adaptation algorithms are Source Only (ResNet), Random (randomly label some target data), BvSB ~\cite{BvSB_2009_CVPR}, Entropy~\cite{entropy_2014_IJCNN}, CoreSet~\cite{CoreSet_2019_ICLR}, WAAL~\cite{WAAL_2020_Shui}, BADGE~\cite{BADGE_2020_ICLR}, AADA~\cite{AADA_WACV}, DBAL~\cite{DBAL_2021_arxiv}, TQS~\cite{Fu_2021_CVPR}, CLUE~\cite{CULE_2021_ICCV}, AdaptSegNet~\cite{AdaSegNet}, and PLCA~\cite{PLCA_2020_NIPS}. Notably, we carry out all experiments with five trials and report the average accuracy. More details are presented in Appendix.

\begin{table*}[!htbp]
    \centering
    \caption{Comparison results on VisDA-2017 and Office-Home with 5\% target samples as the labeling budget.}
    \label{table:home_5_percent}
    \resizebox{\textwidth}{!}{
        \setlength{\tabcolsep}{1.0mm}{
        \begin{tabular}{l|c|ccccccccccccc}
            \hline
            \multirow{2}{*}{Method} & \multirow{2}{*}{VisDA-2017} & \multicolumn{13}{c}{Office-Home} \\
            & & Ar$\to$Cl & Ar$\to$Pr & Ar$\to$Rw & Cl$\to$Ar & Cl$\to$Pr & Cl$\to$Rw & Pr$\to$Ar & Pr$\to$Cl & Pr$\to$Rw & Rw$\to$Ar & Rw$\to$Cl & Rw$\to$Pr & Mean \\
            \hline
            Source Only & 44.7  $\pm$  0.1 & 42.1 & 66.3 & 73.3 & 50.7 & 59.0 & 62.6 & 51.9 & 37.9 & 71.2 & 65.2 & 42.6 & 76.6 & 58.3  \\
            Random & 78.1 $\pm$ 0.6 & 52.5 & 74.3 & 77.4 & 56.3 & 69.7 & 68.9 & 57.7 & 50.9 & 75.8 & 70.0 & 54.6 & 81.3 & 65.8  \\
            BvSB & 81.3 $\pm$ 0.4 & 56.3 & 78.6 & 79.3 & 58.1 & 74.0 & 70.9 & 59.5 & 52.6 & 77.2 & 71.2 & 56.4 & 84.5 & 68.2  \\
            Entropy & 82.7 $\pm$ 0.3 & 58.0 & 78.4 & 79.1 & 60.5 & 73.0 & 72.6 & 60.4 & 54.2 & 77.9 & 71.3 & 58.0 & 83.6 & 68.9  \\
            CoreSet & 81.9 $\pm$ 0.3 & 51.8 & 72.6 & 75.9 & 58.3 & 68.5 & 70.1 & 58.8 & 48.8 & 75.2 & 69.0 & 52.7 & 80.0 & 65.1 \\
            WAAL & 83.9 $\pm$ 0.4 & 55.7 & 77.1 & 79.3 & 61.1 & 74.7 & 72.6 & 60.1 & 52.1 & 78.1 & 70.1 & 56.6 & 82.5 & 68.3 \\
            BADGE & 84.3 $\pm$ 0.3 & 58.2 & 79.7 & 79.9 & 61.5 & 74.6 & 72.9 & 61.5 & 56.0 & 78.3 & 71.4 & 60.9 & 84.2 & 69.9  \\
            \hline
            AADA & 80.8 $\pm$ 0.4 & 56.6 & 78.1 & 79.0 & 58.5 & 73.7 & 71.0 & 60.1 & 53.1 & 77.0 & 70.6 & 57.0 & 84.5 & 68.3  \\
            DBAL & 82.6 $\pm$ 0.3 & 58.7 & 77.3 & 79.2 & 61.7 & 73.8 & 73.3 & 62.6 & 54.5 & 78.1 & 72.4 & 59.9 & 84.3 & 69.6  \\ 
            TQS & 83.1 $\pm$ 0.4 & 58.6 & 81.1 & 81.5 & 61.1 & 76.1 & 73.3 & 61.2 & 54.7 & 79.7 & 73.4 & 58.9 & 86.1 & 70.5  \\
            CLUE & 85.2 $\pm$ 0.4 & 58.0 & 79.3 & 80.9 & 68.8 & 77.5 & 76.7 & 66.3 & 57.9 & 81.4 & 75.6 & 60.8 & 86.3 & 72.5  \\
            \hline
            \bf EADA & \textbf{88.3 $\pm$ 0.1} & \textbf{63.6} & \textbf{84.4} & \textbf{83.5} & \textbf{70.7} & \textbf{83.7} & \textbf{80.5} & \textbf{73.0} & \textbf{63.5} & \textbf{85.2} & \textbf{78.4} & \textbf{65.4} & \textbf{88.6} & \textbf{76.7} \\
            \hline
        \end{tabular}
        }
        }
\end{table*}
\begin{figure}[!htbp]
  \centering  
    \subfigure[ResNet-18]{
      \label{Fig_VisDA2017_res18}
      \includegraphics[width=0.227\textwidth]{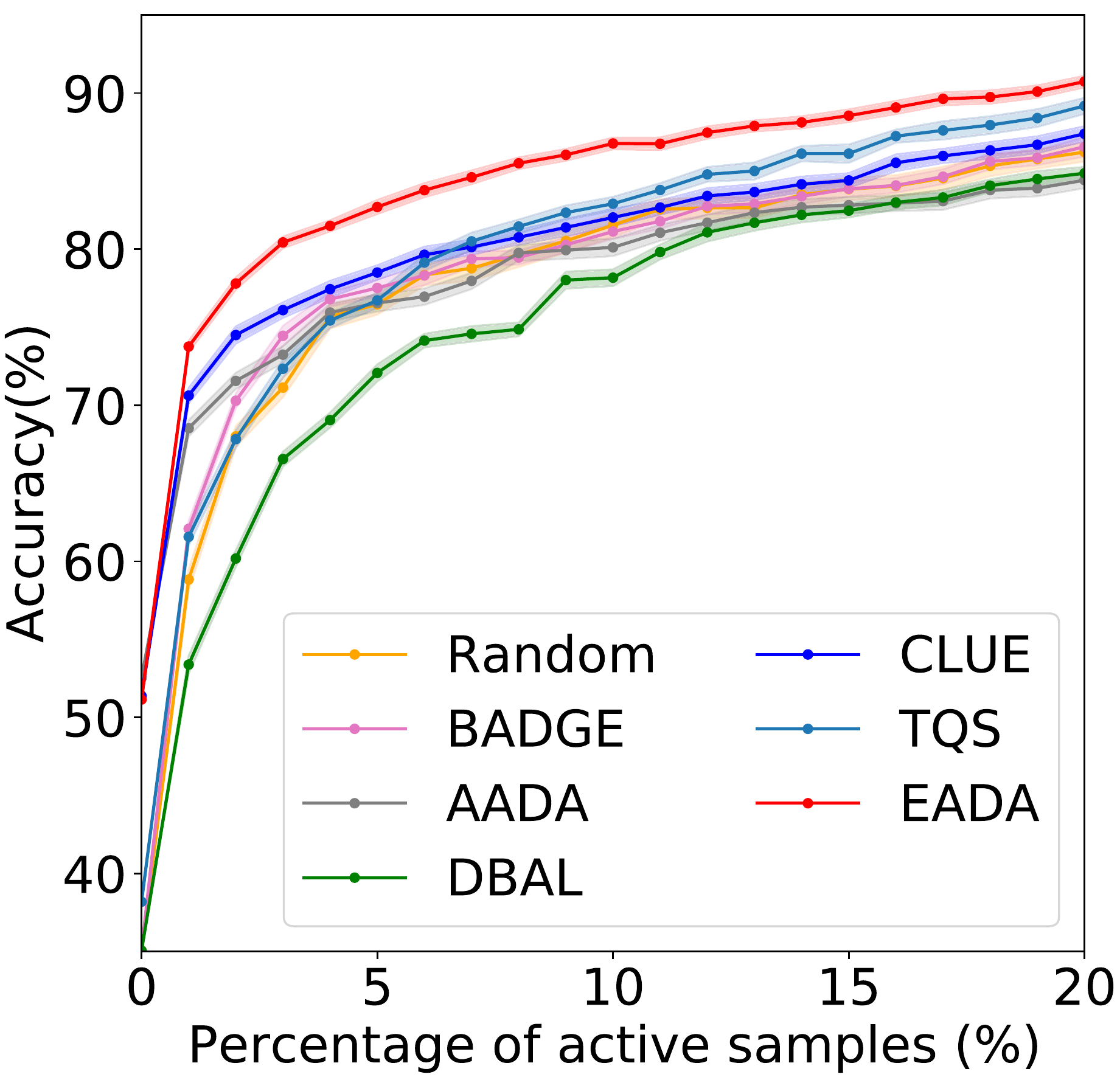}
    }\subfigure[ResNet-50]{
      \label{Fig_VisDA2017_res50}
      \includegraphics[width=0.227\textwidth]{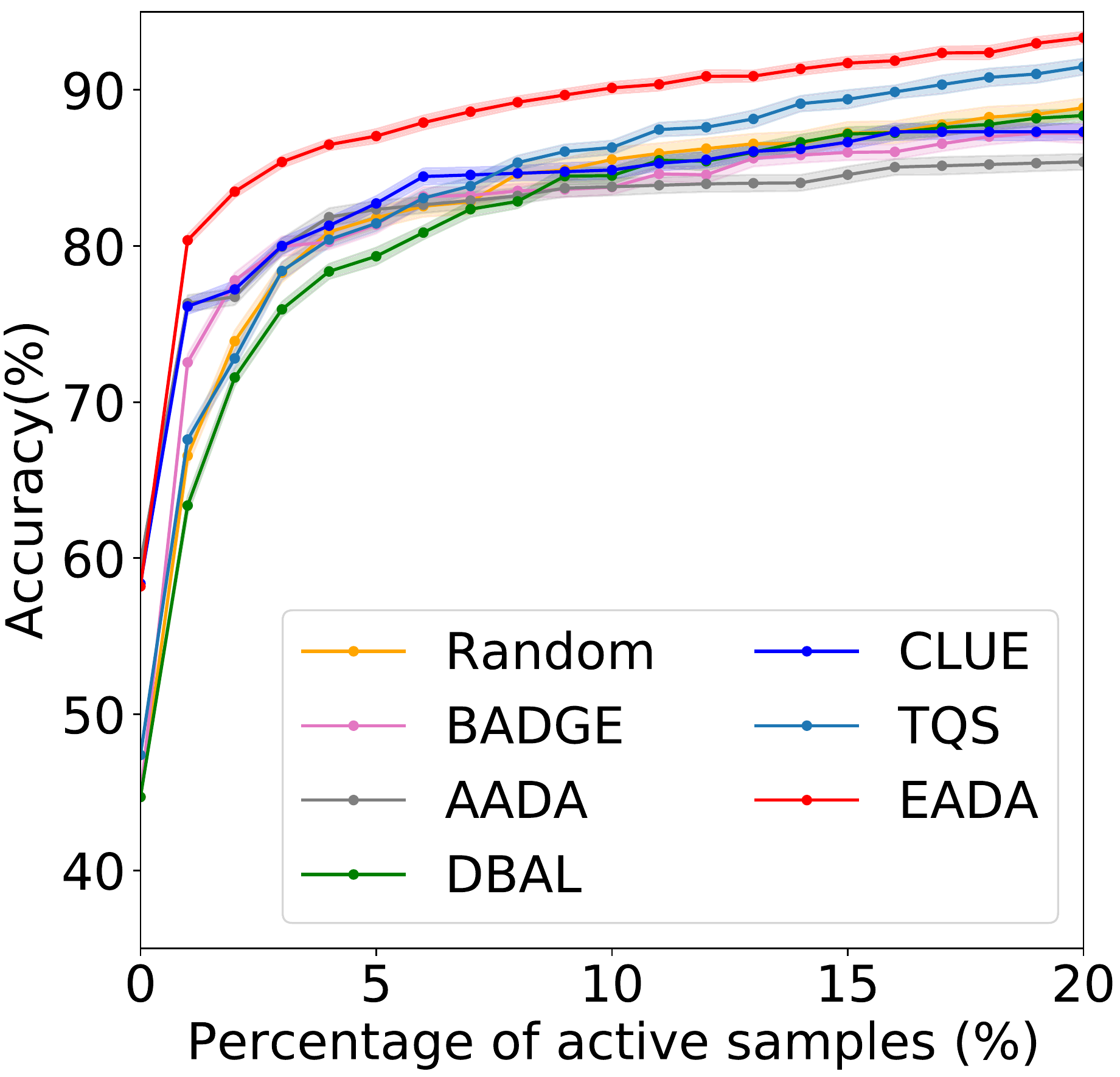}
    }
  \caption{Comparison results of varying the percentage of labeled target samples on VisDA-2017 with ResNet-18/50.}
  \label{Fig_VisDA2017}
\end{figure}

\subsection{Main Results}
\subsubsection{VisDA-2017.} The experimental results of different methods with $5\%$ labeling budget on VisDA-2017 are shown in the first column in Table \ref{table:home_5_percent}, proving that EADA is superior to all the baselines. Randomly selecting samples achieves better performance than ResNet, which implies that active learning is a promising and complementary solution for DA. 

In addition, to further validate the effectiveness of EADA, we vary the target labeling budget from $0\%$ to $20\%$ with different backbones ResNet-18/50 and report the performance after each round in Fig. \ref{Fig_VisDA2017}. We can observe that EADA consistently outperforms alternative methods across rounds. For instance, with shallower ResNet-18, we improve upon the state-of-the-art method, i.e., TQS by 2-6\% over rounds, and obtain comparable results against other methods using deeper ResNet-50 at some rounds. This demonstrates that EADA can indeed select more representative and informative target data using our novel energy-based criterion. Additional comparison results with standard active learning methods are shown in Appendix.

\begin{table}[t]
    \centering
    \caption{Comparison results on Office-31 with 5\% target samples as the labeling budget.}
    \label{table:office_5_percent}
    \resizebox{0.48\textwidth}{!}{
        \setlength{\tabcolsep}{0.5mm}{
        \begin{tabular}{l|ccccccc}
            \hline
            Method & A$\to$D & A$\to$W & D$\to$A & D$\to$W & W$\to$A & W$\to$D  & Mean \\
            \hline
            Source Only & 81.5 & 75.0 & 63.1 & 95.2 & 65.7 & 99.4 & 80.0 \\
            Random & 87.1 & 84.1 & 75.5 & 98.1 & 75.8 & 99.6 & 86.7 \\
            BvSB & 89.8 & 87.9 & 78.2 & 99.0 & 78.6 & \textbf{100.0} & 88.9 \\
            Entropy  & 91.0 & 89.2 & 76.1 & 99.7 & 77.7 & \textbf{100.0} & 88.9 \\
            CoreSet & 82.5 & 81.1 & 70.3 & 96.5 & 72.4 & 99.6 & 83.7 \\
            WAAL & 88.4 & 89.6 & 76.4 & \textbf{100.0} & 76.0 & \textbf{100.0} & 88.4  \\
            BADGE & 90.8 & 89.1 & 79.8 & 99.6 & 79.6 & \textbf{100.0} & 89.8 \\
            \hline
            AADA & 89.2 & 87.3 & 78.2 & 99.5 & 78.7 & \textbf{100.0} & 88.8 \\
            DBAL & 88.2 & 88.9 & 75.2 & 99.4 & 77.0 & \textbf{100.0} & 88.1 \\
            TQS & 92.8 & 92.2 & 80.6 & \textbf{100.0} & 80.4 & \textbf{100.0} & 91.1 \\
            CLUE & 92.0 & 87.3 & 79.0 & 99.2 & 79.6 & 99.8 & 89.5 \\
            \hline
            \bf EADA & \textbf{97.7} & \textbf{96.6} & \textbf{82.1} & \textbf{100.0} & \textbf{82.8} & \textbf{100.0} & \textbf{93.2} \\
            \hline
        \end{tabular}
        }
        }    
\end{table}

\subsubsection{Office-Home \& Office-31.} The results on Office-Home and Office-31 are reported in Table \ref{table:home_5_percent} \& \ref{table:office_5_percent}, respectively, showing the best performance across all tasks. Most Active DA methods generally outperform the traditional AL methods since the latter does not take the domain shift into account. EADA performs much better than all the baselines with a large margin, especially for hard tasks e.g., Ar$\to$Cl, Pr$\to$Cl, D$\to$A and W$\to$A, which emphasizes the benefit of jointly capturing domain characteristic and instance uncertainty for sampling along with free energy regularization.  

\begin{figure}[t]
    \subfigure[GTAV $\to$ Cityscapes]{
      \label{Fig_combine_da}
      \includegraphics[width=0.227\textwidth]{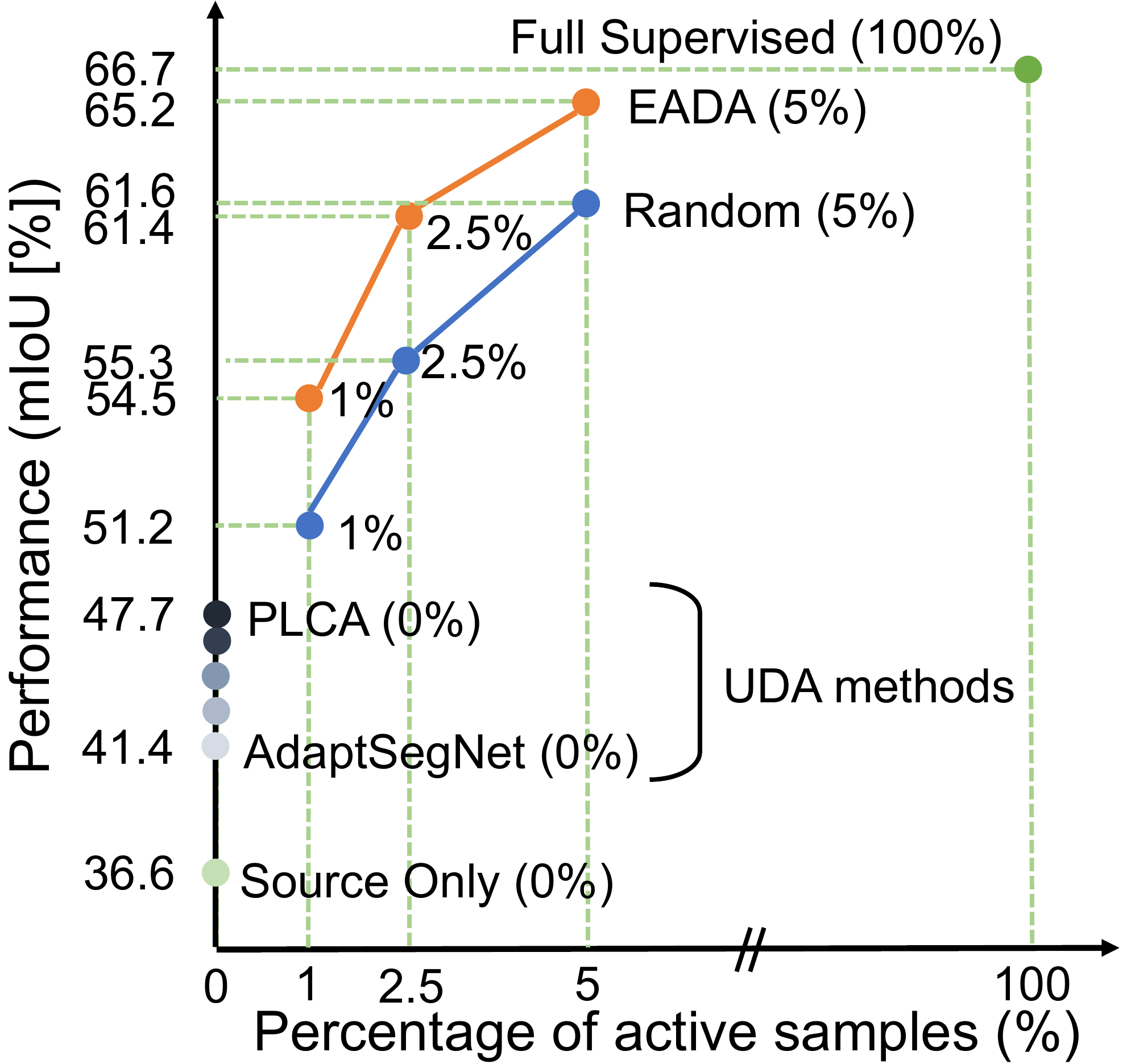}
    }\subfigure[Ablation study]{
      \label{Fig_ablation}
      \includegraphics[width=0.227\textwidth]{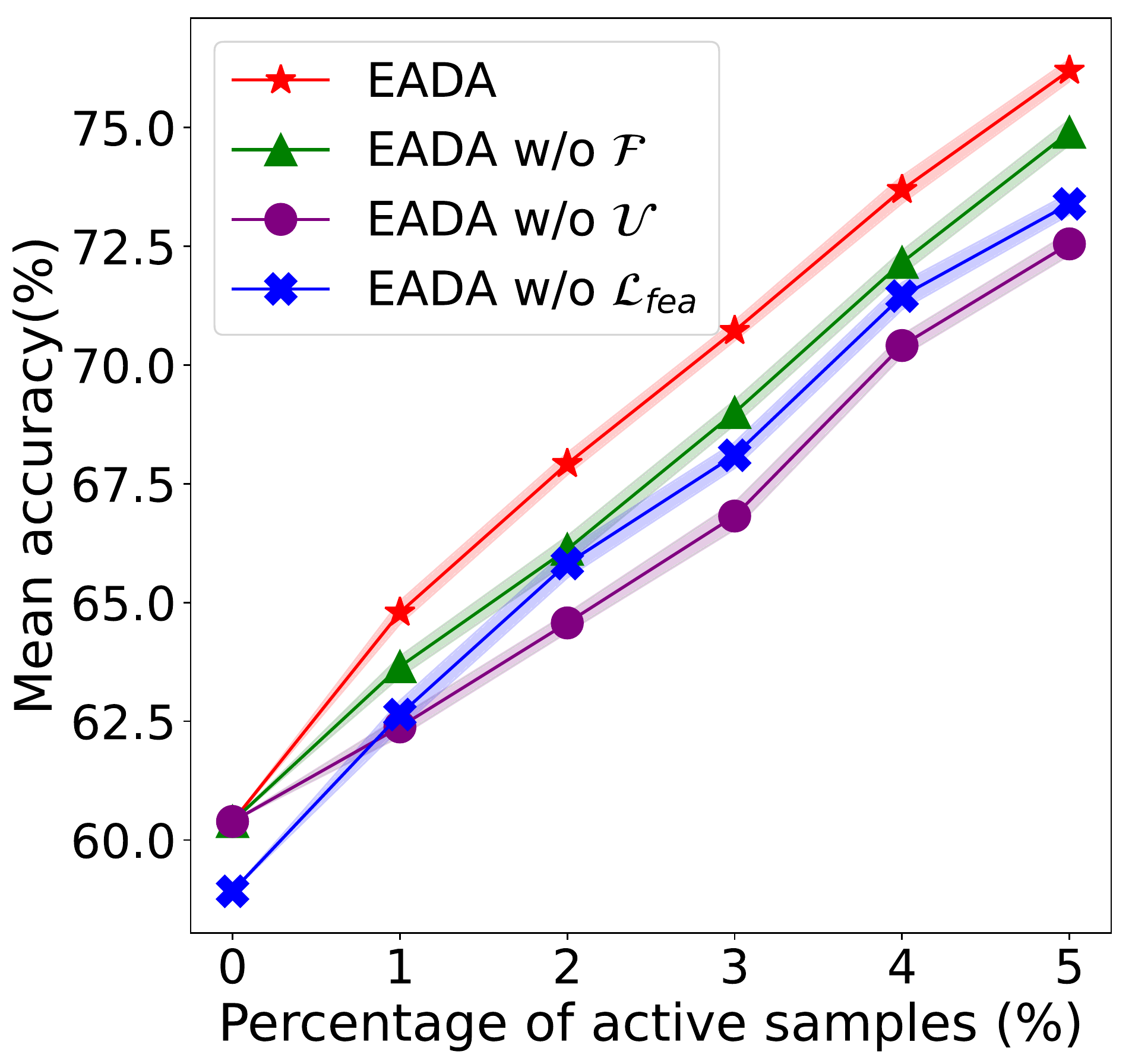}
    }
  \caption{(a) Experimental results on GTAV$\to$Cityscapes.(b) Mean accuracy of EADA and its variants on Office-Home.}
\end{figure}
\subsubsection{GTAV $\to$ Cityscapes.}
While prior works restrict their task to image classification, it is important to also study Active DA in the context of related tasks. 
Here we focus on semantic segmentation adapting from GTAV to Cityscapes and use the same setting as~\cite{AdaSegNet}, which adopts DeepLab-v2~\cite{chen2018deeplab} with ResNet-101 as backbone. We select 5\% target images to query for pixel-level labels of the whole image. The results are shown in Fig.~\ref{Fig_combine_da}. There is a large performance gap between UDA methods and the ``Full Supervised", such as AdaptSegNet, a popular adversarial approach, lags behind 25.2\% mIoU. Surprisingly, EADA brings a significant boost and shows performance comparable to that of fully supervised at the final round.

\subsection{Insight Analysis}
\begin{figure*}[t]
  \centering
  \includegraphics[width=0.98\textwidth]{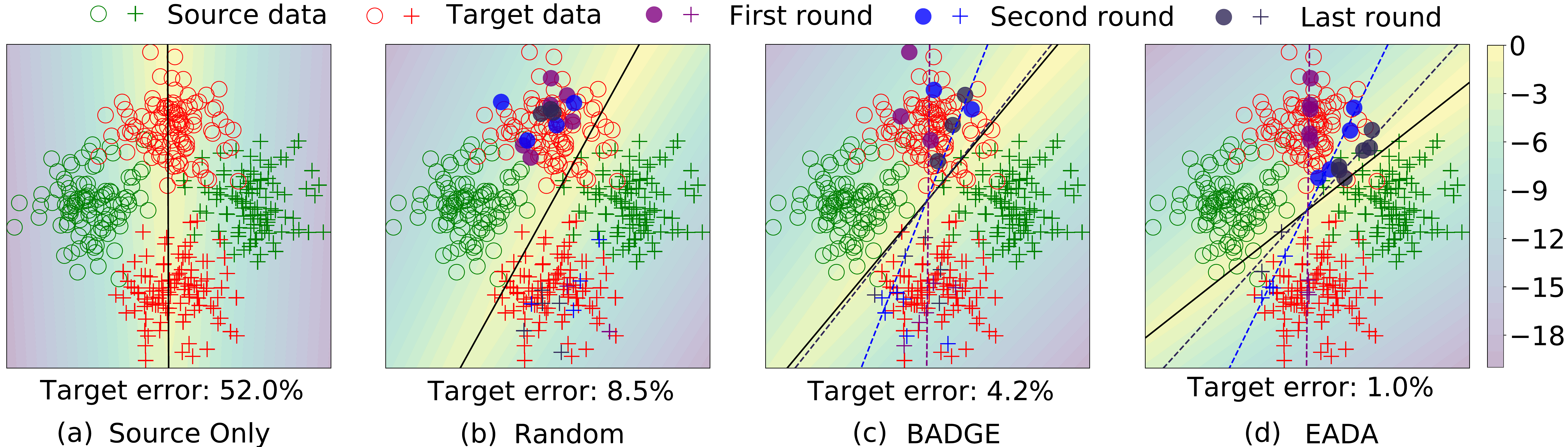}
  \caption{(Best viewed in color.) Illustrative comparison of sampling strategies on a toy example. Red points and green points denote unlabeled target data and labeled source data, respectively. Source data are drawn from two different Gaussian distributions denoted as circle class and plus class and target data are generated by rotating source data directly. We train a single layer fully-connected network and perform 3 rounds of active learning with a per-round budget is 2\% of target samples. In (c) \& (d), we draw the decision boundary before each selection round with a dash line, and the final decision boundary with a solid line.}
  \label{Fig_toy}
\end{figure*}

\subsubsection{Ablation study.} To investigate the efficacy of key components of the proposed EADA, we conduct a thorough ablation study with the following variants on all 12 tasks of Office-Home: (i) EADA w/o $\mathcal{F}$: removing the free energy sampling from selection process; (ii) EADA w/o $\mathcal{U}$: removing the instance uncertainty sampling from the selection process; (iii) EADA w/o $\mathcal{L}_{fea}$: removing $\mathcal{L}_{fea}$ from Eq.~\eqref{eq:overall}.

The results are shown in Fig.~\ref{Fig_ablation}, it is clear that the full method outperforms other variants and achieves large improvements. We also observe that EADA surpasses EADA (w/o $\mathcal{F}$ and w/o $\mathcal{U}$), manifesting that domain characteristic sampling and instance uncertain sampling are both necessary to select representative and informative data. Further, the consistent and notable increases from EADA w/o $\mathcal{L}_{fea}$ to EADA justify our decision to use a regularization term to align free energy distributions between both domains, which is beneficial to reducing the domain shift implicitly.

\subsubsection{Toy example.}
To better explain why the energy-based label acquisition strategy works well and what kind of sample is more representative and informative, we perform a toy example, a binary classification task with domain shift. As shown in Fig.~\ref{Fig_toy}, from the left to the right: Source Only, Random, BADGE, and our EADA are shown one by one and the target errors are 52.0\%, 8.5\%, 4.2\%, 1.0\%, respectively. 

We can make several insightful findings: (i) \textit{Free energy biases}: the values of free energy on target data are considerably higher than those on source data in Fig.~\ref{Fig_toy}(a). To this end, we design a free energy sampling as a surrogate measure to describe domain characteristic. (ii) \textit{Redundant/Trivial selection}: in Fig.~\ref{Fig_toy}(b), we can observe that a large portion of samples selected by Source Only resides in an area where the target data density is high, leading to many redundant instances. BADGE (a state-of-the-art AL method) runs a clustering scheme on ``gradient embedding'' to incorporate both uncertainty and diversity, which slightly mitigates the dilemma of redundancy. However, when we deeply study the relationship between decision boundary and the selected samples in each round, we find that BADGE still selects a few well-aligned samples and the selected samples are not the most uncertain samples of the current classifier. (iii) \textit{Free energy versus decision boundary}: the final decision boundary is the area with the highest free energy. Accordingly, we explore a MvSM metric to precisely quantify the uncertainty of a target sample under the current model. The results in Fig.~\ref{Fig_toy}(d) validate the effectiveness of our method. In short, we define that a target sample with the highest free energy and located around the decision boundary serves as the most valuable, both representative and informative, sample.

\begin{table}[t]
    \centering
    \caption{Effect of selection ratios.}
    \label{table:step_ratio}
    \resizebox{0.48\textwidth}{!}{
        \setlength{\tabcolsep}{0.5mm}{
        \begin{tabular}{l|cccc|cc}
            \hline
            $\alpha_1$ / $\alpha_2$ (\%) & 10 / 10 & 25 / 4 & 50 / 2 & 75 / 1.3 & 100 / 1 & 1 / 100 \\
            \hline
            Office-31 & 91.9 & 92.6 & \textbf{93.2} & 91.5 & 90.8 & 90.4 \\
            Office-Home & 76.2 & 76.3 & \textbf{76.7} & 76.0 & 74.7 & 72.6 \\
            VisDA-2017 & 87.2 & 87.6 & \textbf{88.3} & 87.1 & 86.6 & 85.7 \\
            \hline
        \end{tabular}
        }
        }
\end{table}

\subsubsection{Effect of selection ratios.}
In Table~\ref{table:step_ratio}, we show the accuracy on three image classification benchmarks with varying $\alpha_1$ ($\alpha_2$). Our EADA can achieve consistent performance within a wide range. It is worth noting that excluding any step ($\alpha_1$ or $\alpha_2=100$) will lead to a performance drop. We leave it as future work to explore other more complex combinations like self-adaptive $\alpha_1$ and weighted calculation.

\subsubsection{Time complexity.}
Table \ref{table:time_complexity} lists the query complexity and query time for EADA and comparable baseline methods. BADGE and CLUE achieve better mean accuracy (see Table~\ref{table:home_5_percent} and Table \ref{table:office_5_percent}) but are slower due to a clustering step. Our EADA obtains the best accuracy and is significantly more efficient than the competitive baselines as well.

\begin{table}[t]
    \centering
    \caption{Comparison results on query complexity and query time. $C,M,N$ denote number of classes, source instances and target instances respectively. $D$ denotes feature dimension, $B$ is labeling budget, $T$ denotes clustering rounds.}
    \label{table:time_complexity}
    \resizebox{0.46\textwidth}{!}{
        \setlength{\tabcolsep}{0.9mm}{
        \begin{tabular}{llcc}
            \hline
            & AL        & Query      & Query Time  \\
            & Strategy  & Complexity & (Ar$\to$Cl, VisDA-2017) \\
            \hline
            \multirow{4}{*}{\rotatebox{90}{\centering cluster}}
            & CoreSet & $\mathcal{O} (DN^2)$ & (0.1s, 1.3m) \\
            & BADGE & $\mathcal{O} (CDN^2)$ & (4.7s, 3.5m) \\
            & DBAL & $\mathcal{O} (DN(M+N))$ & (0.4s, 5.3m) \\
            & CLUE & $\mathcal{O} (DN(N+TB))$ & (0.5s, 2.9m) \\
            \hline
            \multirow{3}{*}{\rotatebox{90}{\centering rank}}
            & AADA & $\mathcal{O} (N\log N )$ & (0.03s, 2.2s) \\
            & TQS & $\mathcal{O} (N\log N )$ & (0.04s, 1.7s) \\
            & \bf EADA & $\mathcal{O} (N\log N )$ & \bf (0.02s, 0.9s) \\
            \hline
        \end{tabular}
    }
    }
\end{table}

\section{Conclusion}
In this paper, we present Energy-based Active Domain Adaptation (EADA), an algorithm to tackle performance limitations of domain adaptation at minimal label cost. We propose a novel energy-based sampling strategy into domain adaptation, for the selection of limited target samples that are representative and informative. On top of that, we further explore a regularization term to implicitly diminish the domain gap. In addition, theoretical results about when and why EADA is expected to work are elaborated. Through our experiments, we demonstrate its effectiveness in various transfer scenarios. More generally, our work is but a small step toward alleviating the intensive workload of annotation. This offers encouraging evidence that there remains value to be explored to go beyond the fully supervised method. 


\section{Acknowledgments}
This work was supported in part by the National Natural Science Foundation of China (No. 61902028) and the National Research and Development Program of China (No. 2019YQ1700).

\bibliography{aaai22}

\newpage
\appendix
\section{Appendix}
\paragraph{Contents}
\begin{itemize}
    \item Dataset Details
    \item Implementation Details
    \item Further Analysis for EADA
    \item Additional Results
    \item Detailed Theoretical Proofs
\end{itemize}

\subsection{Dataset Details}

\textbf{VisDA-2017}~\cite{visda2017} is a large-scale synthetic-2-real dataset for image classification competition. In total there are over 280k images from 12 categories. The images are split into three sets, i.e., a training set with 152,397 synthetic 2D renderings of 3D models, a validation set with 55,388 real images, and a test set with 72,372 real images. In this paper, we utilize the training and validation images as the source domain and the target domain, respectively.

\noindent\textbf{Office-Home}~\cite{Office-Home} is a challenging benchmark, consisting of 15,500 images in 65 object classes. There are 4 extremely distinct domains: Artistic images ({\bf Ar}), Clip Art ({\bf Cl}), Product images ({\bf Pr}), and Real-World images ({\bf Rw}). And we build twelve transfer tasks: Ar$\to$Cl, Ar$\to$Pr, $\dots$, Rw$\to$Cl, Rw$\to$Pr.

\noindent\textbf{Office-31}~\cite{Office31} is widely adopted by domain adaptation methods, comprising 31 categories of 4,110 images. It involves 3 different domains: \textbf{A}mazon (images downloaded from Amazon website), \textbf{D}SLR (images collected from digital SLR camera) and \textbf{W}ebcam (images recorded by web camera). We evaluate our method on 6 transfer tasks: A$\to$D, A$\to$W, $\dots$, W$\to$D.

\noindent\textbf{GTAV}~\cite{GTA5} contains 24,966 synthetic images with the resolution of 1914$\times$1052, which are rendered using the open-world video game ``Grand Theft Auto V". There are 19 semantic categories that are compatible with the semantic categories in the Cityscapes dataset.

\noindent\textbf{Cityscapes}~\cite{cityscapes} includes 5,000 urban scene images of resolution 2048$\times$1024. They are split into a training set with 2,975 images and a validation set with 500 images. Similar to~\cite{AdaSegNet,PLCA_2020_NIPS}, we evaluate our model on the validation set and report the mIoU of the common 19 classes.

\subsection{Implementation Details}
\subsubsection{Image classification.} We implement our experiments on the widely-used PyTorch~\cite{paszke2019pytorch} platform. For a fair comparison, our backbone network is identical to the competitive methods and is also pre-trained on ImageNet~\cite{imagenet}. For optimizer, we use the AdaDelta with a learning rate of 0.1 and train for 50 epochs. The batch size is 32. We adopt a unified set of hyper-parameters throughout the VisDA-2017, Office-Home, Office-31 datasets, where $\gamma$=0.01, $\alpha_1$=50, and $\alpha_2$=2. We use random-crop images for data augmentation during training and use center-crop images for testing. We follow the standard protocol~\cite{Fu_2021_CVPR,AADA_WACV,CULE_2021_ICCV} to use the whole target domain as testing data. We carry out experiments with five different random seeds and report the average classification accuracy.

\subsubsection{Semantic segmentation.} We employ the DeepLab-v2~\cite{chen2018deeplab} as the feature extractor which is composed of the backbone ResNet-101~\cite{resnet} pre-trained on ImageNet~\cite{imagenet} and the Atrous Spatial Pyramid Pooling (ASPP) module. 
Following~\cite{AdaSegNet}, sampling rates of ASPP module are fixed as \{6, 12, 18, 24\}. 
To train the segmentation network, we adopt the SGD optimizer where the momentum is 0.9 and the weight decay is 10$^{-4}$. 
The learning rate is initially set to 2.5$\times$10$^{-4}$ and is decreased following a `poly' learning rate policy with a power of 0.9. 
$\gamma$ is constantly set to 0.001 and $\alpha_1, \alpha_2$ are set to 50 and 2, respectively. 
The source input image is resized to 1280$\times$720 and the target input image is resized to 1024$\times$512. We respectively select 1\%, 2.5\%, and 5\% Cityscapes training samples as total labeling budget and query for pixel-level semantic annotation of the whole image.

\subsubsection{Labeling budget.} Following the previous active domain adaptation work~\cite{Fu_2021_CVPR}, the labeling budge in each selection round is \textbf{1\% of all target samples}. In Table~\ref{table:home_5_percent}, Table~\ref{table:office_5_percent} of the main paper, we perform 5 rounds (in total 5\% target samples) for VisDA-2017, Office-Home, Office-31 and report the classification accuracy of final model. Similarly, we perform 20 rounds (in total 20\% target samples) for VisDA-2017 and provide the accuracy after each round in Fig.~\ref{Fig_VisDA2017} of the main paper. Unless specified otherwise, we perform 5 round selections throughout the analysis below.

\begin{figure*}[t]
    \centering  
    \includegraphics[width=0.98\textwidth]{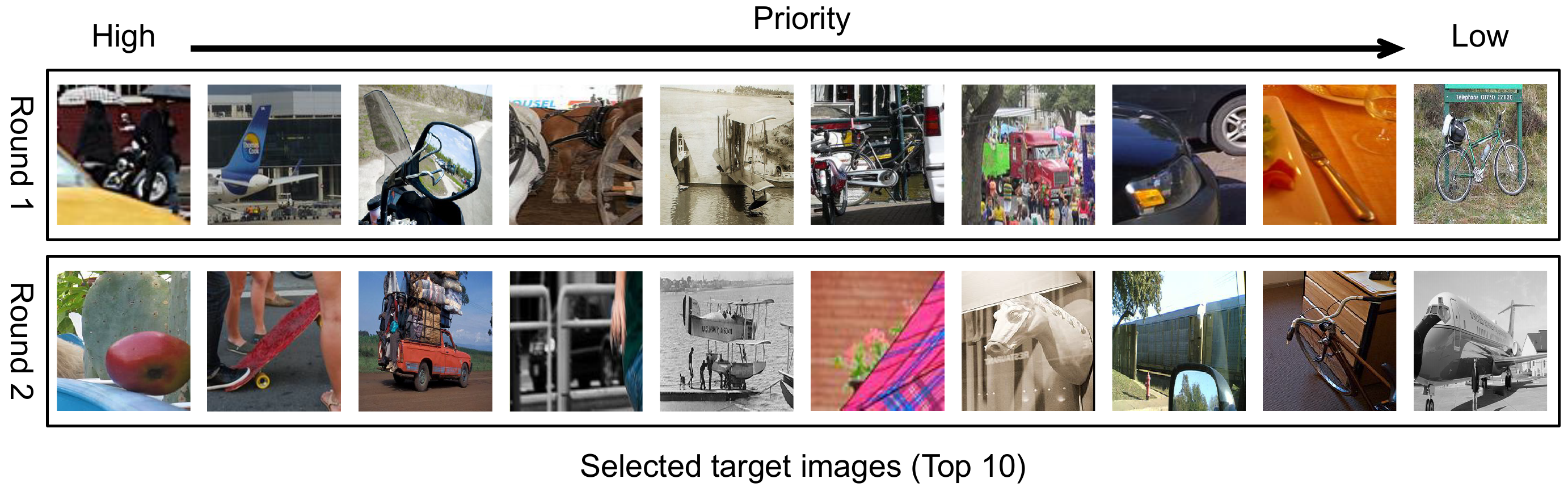}
    \caption{Visualization of instances selected by our method at Round 1 and Round 2 on VisDA-2017.}
    \label{fig:visulization}
\end{figure*}
\subsubsection{Baseline implementation.} Note that partial reported results are copied from~\cite{Fu_2021_CVPR} if the experimental setup is the same. We now elaborate on our implementation of other baseline algorithms:
\begin{itemize}
    \item Random: The naive baseline of randomly selecting 1\% target samples to query at each round.
    \item BvSB ~\cite{BvSB_2009_CVPR}: We compute the difference between the largest and second-largest predicted probability as an uncertainty metric and then select the bottom 1\% samples sorted according to the value of this metric at each round.
    \item Entropy~\cite{entropy_2014_IJCNN}: We utilize the entropy of the model outputs as a confidence measurement for each unlabeled target sample and then select the top 1\% samples sorted according to the value of this measurement at each round.
    \item CoreSet~\cite{CoreSet_2019_ICLR}: A representative sampling algorithm using core-set selection. We implement CoreSet using
    the released code: \url{https://github.com/ozansener/active_learning_coreset}.
    \item WAAL~\cite{WAAL_2020_Shui}: A hybrid sampling algorithm which models the interactive procedure in active learning as distribution matching by adopting Wasserstein distance. We implement WAAL using the released code: \url{https://github.com/cjshui/WAAL}
    \item BADGE~\cite{BADGE_2020_ICLR}: We compute ``gradient embeddings'' through taking the gradient of model loss w.r.t. classifier weights. Next, the $k$-MEANS++ scheme~\cite{kmeans_add} is run on these embeddings to yield a bach of samples. We implement BADGE using the released code: \url{https://github.com/JordanAsh/badge}
    \item AADA~\cite{AADA_WACV}: In AADA, a domain discriminator $G_d$ is learned to distinguish the features obtained from the feature extractor $G_f$ whether come from the source domain or the target domain. For active sampling strategy, all unlabeled target data are scored via the selection criterion $s(x)$ defined in the original paper: $s(x)=\frac{1-G_d^*(G_f(x))}{G_d^*(G_f(x))}\mathcal{H}(G_y(G_f(x)))$, where $G_y$ is the class predictor and $\mathcal{H}$ denotes the model entropy. Then, we select top 1\% samples for labeling at each round. 
    \item DBAL~\cite{DBAL_2021_arxiv}: A discrepancy-based active learning for domain adaptation. We implement DBAL using the released code: \url{https://github.com/antoinedemathelin/dbal}
    \item TQS~\cite{Fu_2021_CVPR} A state-of-the-art active domain adaptation method, which selects the most informative target samples by an ensemble of transferable committee, transferable uncertainty, and transferable domainness. We implement TQS using the released code \url{https://github.com/thuml/Transferable-Query-Selection}
    \item CLUE~\cite{CULE_2021_ICCV}: A more recent example that jointly captures uncertainty and diversity for active domain adaptation. We implement CLUE according to Algorithm 1 provided in~\cite{CULE_2021_ICCV}. Consider with the original work, we also optimize a semi-supervised adversarial entropy loss and perform cross-validation on source data to tune the hyperparameters.
\end{itemize}

\begin{figure}[t]
	\begin{minipage}[b]{0.5\linewidth} 
		\centering
        \includegraphics[width=0.9\linewidth]{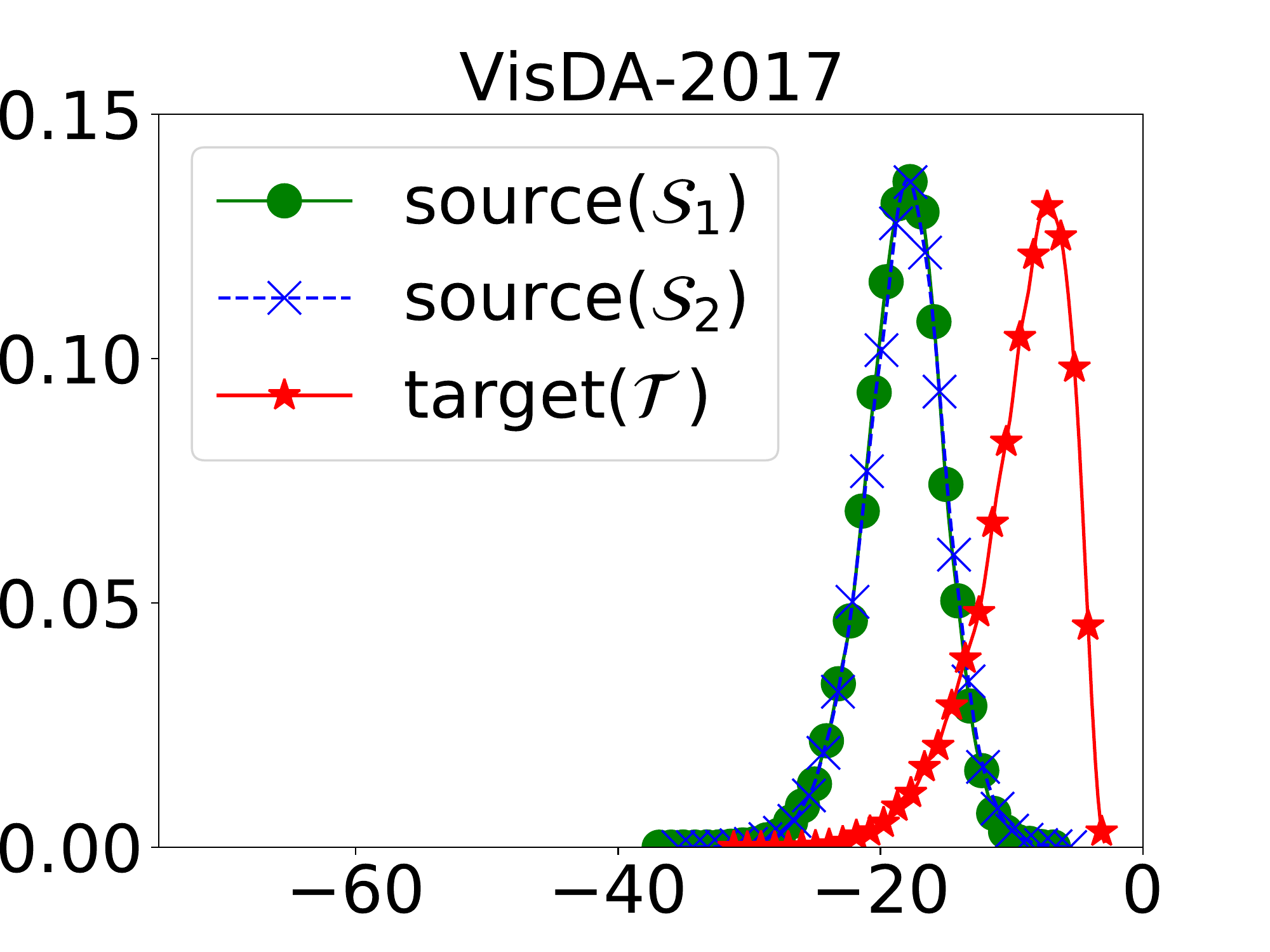}
		\caption{Free energy biases}\label{fig:bias}
	\end{minipage}%
	\begin{minipage}[b]{0.5\linewidth} 
		\centering
        \includegraphics[width=0.9\linewidth]{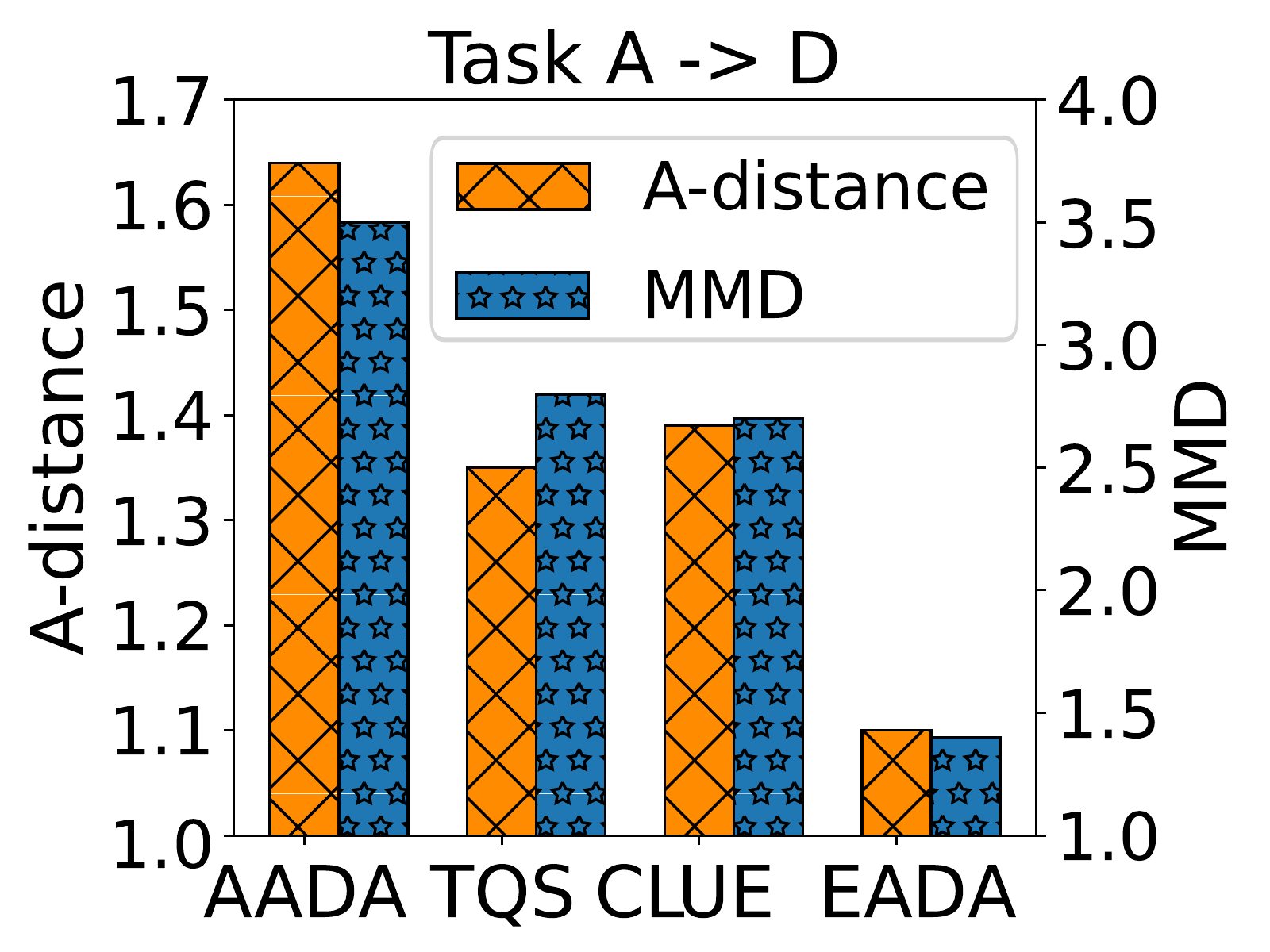}
		\caption{$\mathcal{A}$-distance}\label{fig:a_distance}
	\end{minipage}
\end{figure} 
\subsection{Further Analysis for EADA}
\subsubsection{Visualizing selected sampels.} To understand the behavior of EADA intuitively, we visualize the top 10 samples selected by EADA at Round 1 and Round 2 on VisDA-2017 in Fig.~\ref{fig:visulization}. As seen, EADA manages to sample instances from every class, and more representative and informative target samples are selected.

\subsubsection{More details about free energy biases.} As a matter of fact, free energy biases exhibit when two domains have distinct distributions, i.e., $p_s(x) \neq p_t(x)$. Statistically, Eq.~\eqref{eq:px_free_energy} indicates that $\mathcal{F}(x)$ could be substituted for $p(x)$ where $\mathcal{F}(x)$ is the free energy of $x$ from our model. $\mathcal{P}_{s}(\mathcal{F})$ and $\mathcal{P}_{t}(\mathcal{F})$ are free energy distributions of source and target, respectively. Thus, if $p_s(x) \neq p_t(x)$, $\mathcal{P}_{s}(\mathcal{F}) \neq \mathcal{P}_{t}(\mathcal{F})$. Empirically, we randomly divide source domain into two data sets ($\mathcal{S}_1\,, \mathcal{S}_2$), where $\mathcal{S}_1$ is used for supervised training while both $\mathcal{S}_2$ and $\mathcal{T}$ are used for testing. Then, we plot the free energy distributions of these three sets in Fig.~\ref{fig:bias}. The results show that $\mathcal{S}_1$ (green, circle) and $\mathcal{T}$ (red, star) still exhibit biases while $\mathcal{S}_1$ (green, circle) and $\mathcal{S}_2$ (blue, cross) do not.

\subsubsection{Effect of free energy alignment.} Our key observation is that the values of free energy on target samples are considerably higher than those on source ones, called \textit{free energy biases}. Mathematically, we show that these biases exhibit when two domains have different distributions. Naturally, one can treat this as a surrogate to reflect the divergence across domains. By designing a simple regularization term, we are allowed to minimize the biases, which to some extent aligns the distribution of source and target data. Experimentally, by calculating two common measurements, $\mathcal{A}$-distance and MMD, on task A $\to$ D (see Fig.~\ref{fig:a_distance}), we observe that EADA achieves a lower $\mathcal{A}$-distance/MMD, implying lower domain divergence.

\begin{figure}[t]
    \centering  
        \subfigure[Office-Home: Rw $\to$ Pr]{
        \includegraphics[width=0.227\textwidth]{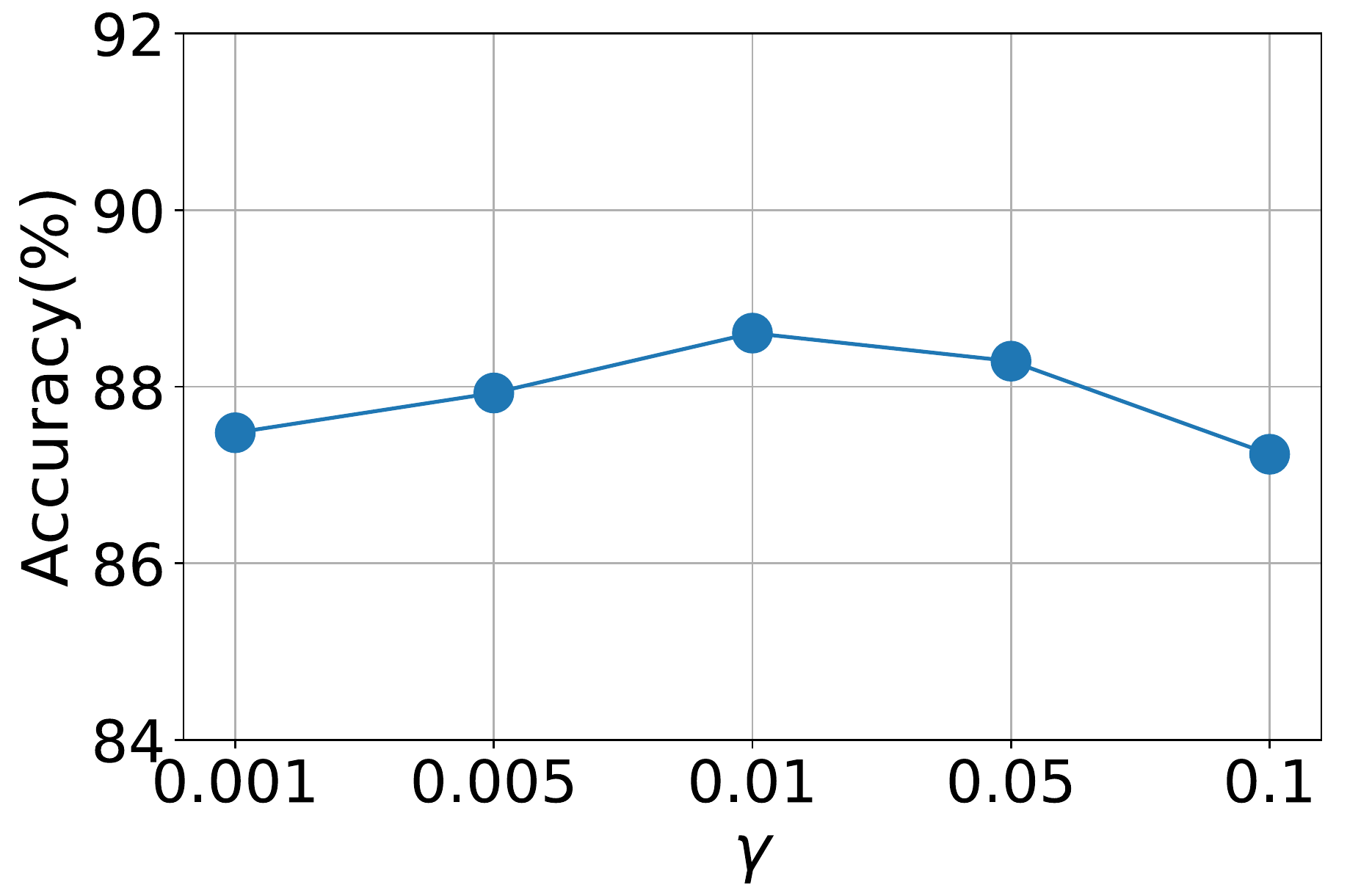}
        }\subfigure[Office-31: A $\to$ D]{
        \includegraphics[width=0.227\textwidth]{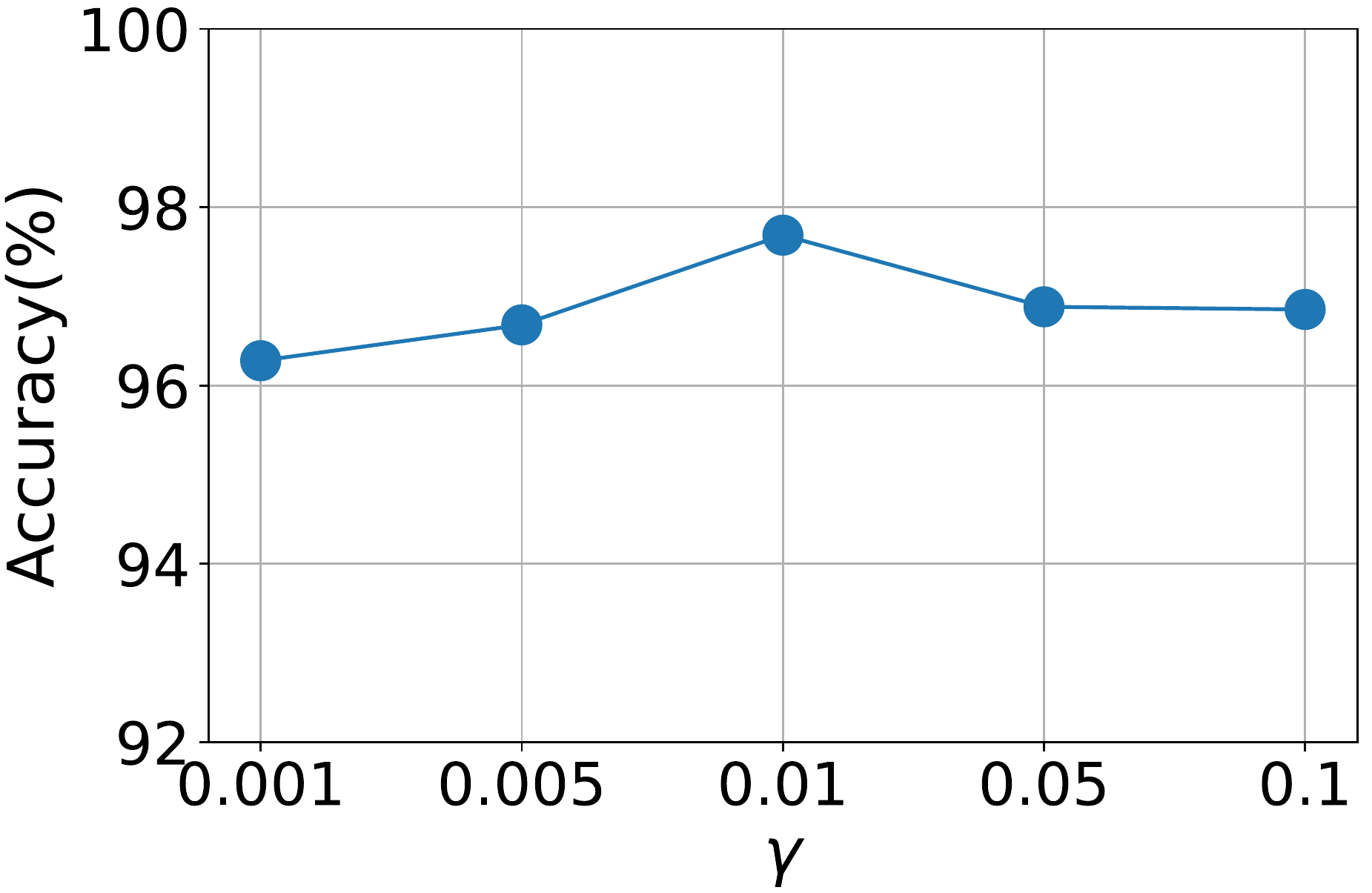}
        }
    \caption{Effect of regularization weight $\gamma$.}
    \label{fig:regularization_weight}
\end{figure}
\subsubsection{Effect of regularization weight $\gamma$.}
We perform parameter sensitivity analysis to evaluate the sensitivity of EADA on task Rw $\to$ Pr of Office-Home and A $\to$ D of Office-31. As shown in Fig.~\ref{fig:regularization_weight}, we select regularization weight from $\gamma\in\{0.001, 0.005, 0.01, 0.05, 0.1\}$. It can be seen that the performance of EADA increases first and then decreases slightly, which is a slow bell-shaped curve. Overall, these results indicate that the robustness of EADA under a wide range of parameter choices.

\begin{table}[t]
    \centering
    \caption{Error rate of the selected samples in each round on VisDA-2017 with \textbf{5\%} target samples as the labeling budget.}
        \label{table:error_rate}
        \begin{tabular}{l|ccccc}
            \hline
            Method / Round  & 1 & 2 & 3 & 4 & 5 \\
            \hline
            AADA & 83.6 & 79.2 & 75.5 & 75.3 &  77.3 \\
            CLUE & 36.8 & 26.5 & 26.2 & 24.0 &  22.2 \\
            \hline
            \bf EADA & 70.2 & 67.3 & 67.1 & 64.6 & 66.5 \\
            \hline
        \end{tabular}
\end{table}
\subsubsection{Error rate of the selected samples.} To explore the relationship between sampling strategy and predictions by the current model, we compare the error rate of the selected samples with two popular sampling strategies in Active DA, i.e., adopting the output of a domain discriminator (AADA) or using the distance to the cluster centroids (CLUE), in Table~\ref{table:error_rate}. Interestingly, the error rate of selected samples decreases as the number of rounds increases. We conjecture that the accuracy and generalization of the model will be significantly improved when a limited number of target samples are labeled and used to train the model. However, contrary to intuition, we find that selecting more incorrect target samples may not lead to performance gains in the target domain. These results suggest that a good active selection strategy should identify those samples that are representative and informative. We see it as future work.

\subsubsection{Analysis of the trade-offs on $\mathcal{L}_{nll}$.}
Given numerous labeled data from the source domain and a small quota of the target domain, we compute the negative log-likelihood loss over all available labeled data. For simplicity, in the main paper, we directly add the two supervised losses in Eq.~\eqref{eq:overall} without trade-offs. However, we argue that the trade-offs may be needed when we identify and annotate a few target samples~\footnote{See CLUE~\cite{CULE_2021_ICCV}, the loss weights for source and target supervised loss are set to 0.1 and 1.0, respectively. }. Here, we employ a scalar weight to re-define the overall objective:
\begin{small}
    \begin{equation}
        \begin{aligned}
        \min_{\theta}~& \delta_s \mathbb{E}_{(x,y)\sim\mathcal{S}}\mathcal{L}_{nll}(x,y;\theta) + \delta_t \mathbb{E}_{(x,y)\sim\mathcal{T}_{l}}\mathcal{L}_{nll}(x,y;\theta) \\ \nonumber
        &+ \gamma \mathbb{E}_{x\sim \mathcal{T}_u}\mathcal{L}_{fea}(x;\theta),
        \label{eq:overall_weighted}
    \end{aligned}
    \end{equation}
\end{small}
where $\delta_s$ and $\delta_t$ are scalar weights. We fix loss weight $\delta_t=1$ and study three ways for $\delta_s$: $\delta_s=1.0$ (default), $\delta_s=0.5$ and $\delta_s=0.1$. Looking at Table~\ref{table:source_supervised}, it is apparent that $\delta_s=0.1$ achieves the best accuracy, verifying the weight of source supervised loss should indeed be small when having access to a few labeled target samples.

\begin{table}[t]
    \centering
    \caption{Analysis of the trade-offs on $\mathcal{L}_{nll}$.}
        \label{table:source_supervised}
        \begin{tabular}{ccc}
            \hline
            $\delta_s=1.0$ (default) &  $\delta_s=0.5$ & $\delta_s=0.1$ \\
            88.0 & 88.3 & \bf 88.7 \\
            \hline
        \end{tabular}
\end{table}

\subsection{Additional Results}
\subsubsection{Varying labeling budget on VisDA-2017.} In Fig.~\ref{Fig_VisDA2017} of the main paper, we vary the target labeling budget from $0\%$ to $20\%$ with existing Active DA methods. For completeness, we present the performance results of other active learning methods in Fig.~\ref{fig:visda2017_active_methods}. EADA consistently beats all previous active learning methods. In addition, due to the existence of domain shift, traditional active learning methods perform poorly on the target domain even few target samples are labeled. We can see that EADA performs best even with a very limited labeling budget ($<$5\%), comparable to other active learning approaches that allow for larger budgets (about 12\%), suggesting that EADA requires very little budget, yet significantly improves adaptation performance.

\begin{figure}[t]
    \centering  
      \subfigure[ResNet-18]{
        \includegraphics[width=0.227\textwidth]{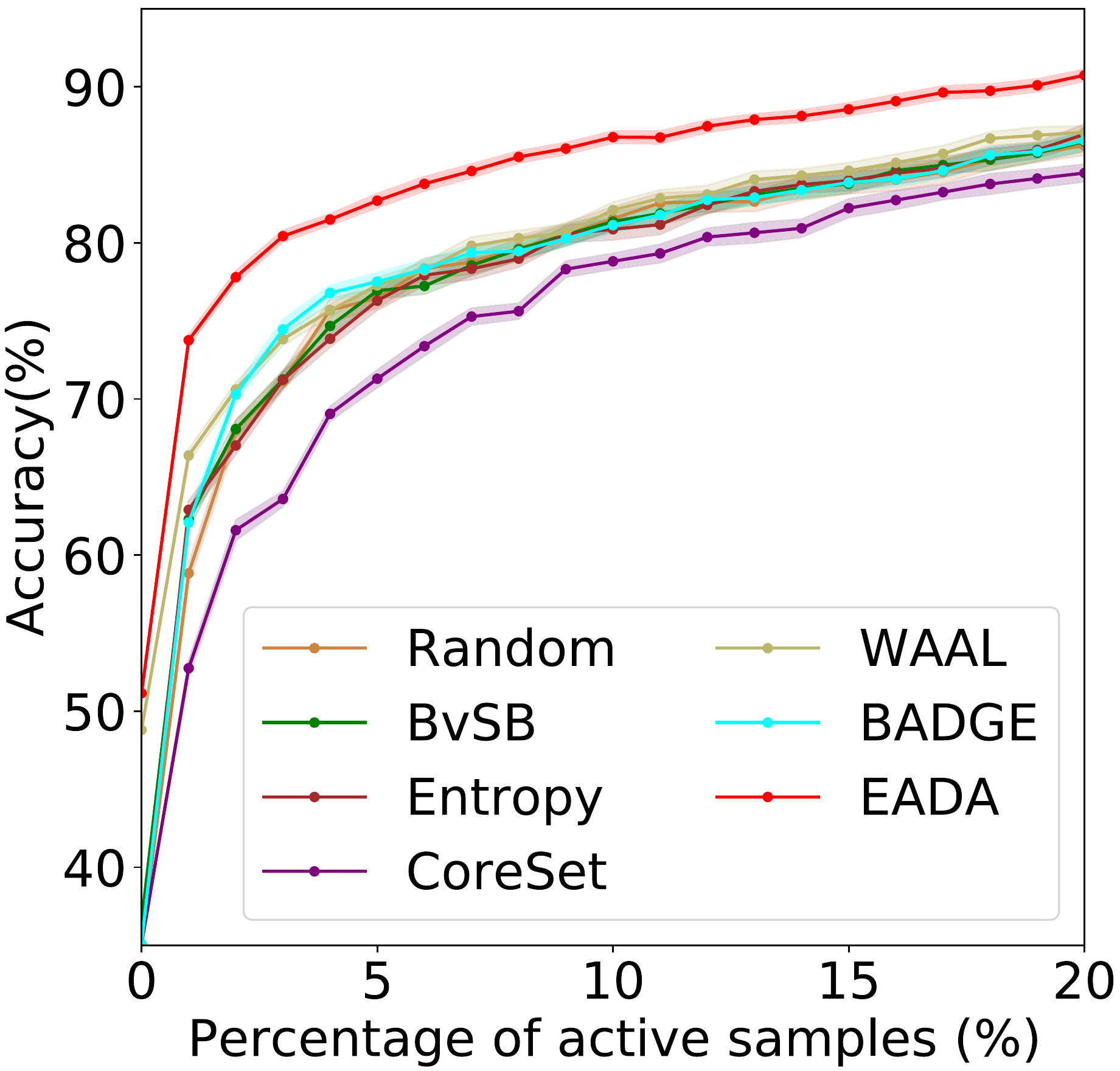}
      }\subfigure[ResNet-50]{
        \includegraphics[width=0.227\textwidth]{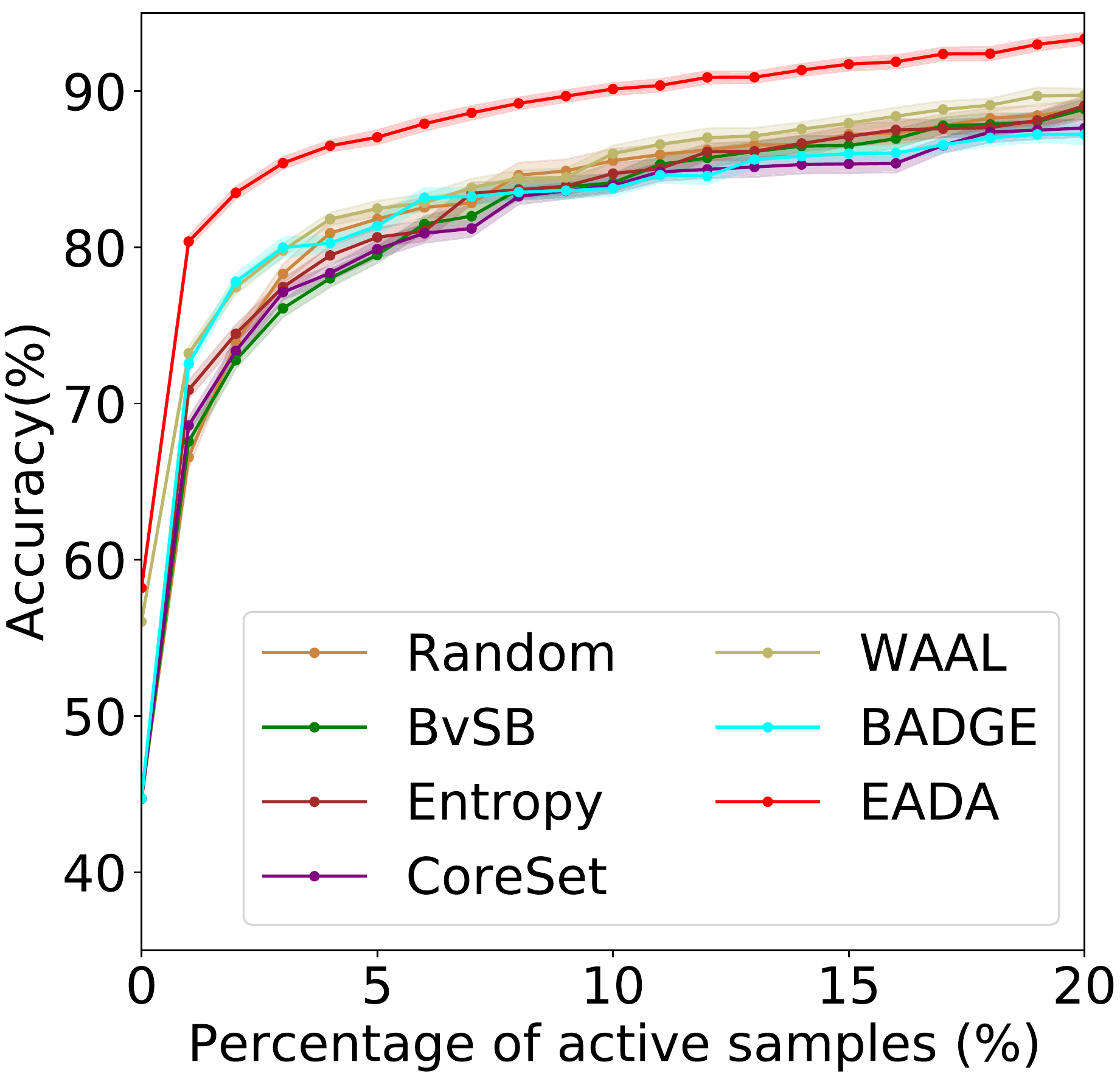}
      }
    \caption{Comparison results of varying the percentage of labeled target samples on \textbf{VisDA-2017} with ResNet-18/50.}
    \label{fig:visda2017_active_methods}
\end{figure}
\subsection{Detailed Theoretical Proofs}
\subsubsection{Gradient analysis about $\mathcal{L}_{nll}$.}
In the method part, we mention that the second term in $\mathcal{L}_{nll}$ (Eq.~\eqref{eq:nll_energy_mine_free_energy} in the main paper) will cause the energies of all answers to be pulled up. The energy of the correct answer is also pulled up, but not as hard as it is pushed down by the first term. In other words, the energy of the correct answer will be pulled down and other energies of incorrect answers will be pulled up when training with the $\mathcal{L}_{nll}$. Here, we analyze the gradient of $\mathcal{L}_{nll}$ to corroborate this statement. 

Firstly, we emphasize the equation of $\mathcal{L}_{nll}$ with $\tau$=1 again:
\begin{align}
        \mathcal{L}_{nll}(x,y;\theta) & = E(x,y;\theta) + \log\sum_{c\in\mathcal{Y}} \exp\left(-E(x,c;\theta)\right) \nonumber\\
        & = E(x, y; \theta) - \mathcal{F}(x;\theta)\,.
\end{align}%
And with the energy function $E(x,y;\theta)$, the conditional probability of label $y$ give the input $x$ can be estimated through the Gibbs distribution:
\begin{equation}
    p(y | x;\theta) = \frac{\exp(-E(x,y;\theta))}{\sum_{c\in\mathcal{Y}}\exp(-E(x,c;\theta))}\,.
\end{equation}%
In backward propagation, gradient of $\mathcal{L}_{nll}(x, y;\theta )$ is given by:
\begin{small}
    \begin{equation}
        \frac{\partial \mathcal{L}_{nll}(x, y;\theta )}{\partial \theta } = \frac{\partial E(x, y;\theta )}{\partial \theta } - \frac{\partial \mathcal{F}(x;\theta )}{\partial \theta} \,.
        \label{eq:gradient_analysis}
    \end{equation}
\end{small}
For the second term in Eq.~\eqref{eq:gradient_analysis}, we have 
\begin{small}
    \begin{equation}
        \begin{aligned}
        \label{eq:gradient_for_free_energy}
        \frac{\partial \mathcal{F}(x;\theta)}{\partial \theta } &= \frac{\partial \left(-\log\sum_{c\in\mathcal{Y}}\left(\exp\left(-E(x,c;\theta)\right)\right) \right)}{\partial \theta } \\
        & = \sum_{c \in \mathcal{Y}} p(c|x;\theta) \frac{\partial(E(x, c;\theta))}{\partial \theta }\,. \\
        \end{aligned}
    \end{equation}
\end{small}%
From above, we can get
\begin{small}
    \begin{align}
        \label{gradient_of_nll}
            \frac{\partial \mathcal{L}_{nll}(x, y;\theta)}{\partial \theta } &= \frac{\partial E(x, y;\theta)}{\partial \theta } - \sum_{c \in \mathcal{Y}} p(c|x;\theta) \frac{\partial(E(x, c;\theta))}{\partial \theta }  \nonumber\\
            & = (1 - p(y | x;\theta)) \frac{\partial E(x, y;\theta)}{\partial \theta }\\
            &~~~~ - \sum_{c \in \mathcal{Y} \backslash \{y\}} p(c | x;\theta) \frac{\partial(E(x, c;\theta))}{\partial \theta } \,. \nonumber
    \end{align}
\end{small}
From Eq.~\eqref{gradient_of_nll}, we can clearly see that the first term 
will pull down the energy of the correct answer and the second term will pull down the energies of all the incorrect answers.

\subsubsection{Toy example} 
Consider a $C$-class classification problem in a simple energy-based model (EBM) where $x \in \mathbb{R}^{D}$ denotes the input, $ y\in \{1,\cdots,C\}$ denotes the label. And the network is a single layer fully-connected network with parameter $\mathbf{W} \in \mathbb{R}^{C \times D}$. The output for $x$ is
\begin{small}
\[
    O\left( x\right)  = \mathbf{W} x ,
\]
\end{small}%
where 
\begin{small}
\[
    \mathbf{W} = \begin{pmatrix}
            \omega_{11} & \cdots & \omega_{1D} \\
            \vdots  & \ddots  & \vdots \\
            \omega_{C1} & \cdots & \omega_{CD} 
        \end{pmatrix}
    = \begin{pmatrix}
        \omega_1 & \cdots & \omega_C
    \end{pmatrix}^{\top}.
\]
\end{small}%
In this discrete prediction task, the negative log-likelihood loss is
\begin{small}
\begin{equation}\label{toy_loss}
    \Ln\left(x, y;\mathbf{W}\right)  = E\left(x, y;\mathbf{W}\right) -\F\left(x;\mathbf{W}\right) ,
\end{equation}
\end{small}%
where $E\left(x, y;\mathbf{W}\right)$ denotes the energy of correct answer and $\F\left(x;\mathbf{W}\right)$ denotes free energy of $x$. And they can be calculated as:
\begin{small}
\begin{align}
    E\left(x,j;\mathbf{W}\right)  &= O\left(x\right) [j] = \omega_{j}^{\top} x, \forall j\in\{1,\cdots,C\}, \label{toy_energy} \\ 
    \F\left( x;\mathbf{W}\right)  & = -\log\sum_{c = 1}^{C} \exp\left( -E\left( x,c;\mathbf{W}\right) \right)  \nonumber\\
    & = -\log\sum_{c = 1}^{C} \exp\left( -\omega_{c}^{\top} x\right) . \label{toy_free_energy}
\end{align}
\end{small}
Following, we calculate the gradient of the negative log-likelihood loss $\Ln(x,y;\mathbf{W})$ step by step.

First of all, for $E\left( x, j;\mathbf{W}\right)$ in Eq.~\eqref{toy_energy}, we have 
\begin{small}
\begin{equation}\label{gradinet_energy}
    \nabla E\left( x, j;\mathbf{W}\right)  = \begin{pmatrix}
        0 & \cdots & 0 \\
        \vdots  & \ddots  & \vdots \\
        x_1 & \cdots & x_D \\ 
        \vdots  & \ddots  & \vdots \\
        0 & \cdots & 0 
    \end{pmatrix}    
    = x{(e^j)}^{\top} ,
\end{equation}
\end{small}
where $e^j \in \mathbb{R}^D$, and $e_i^j = 1$ when $i = j$, 0 otherwise. Then, according to Eq.~\eqref{eq:gradient_for_free_energy}, we have
\begin{small}
\begin{equation}\label{gradient_free_energy}
    \begin{aligned}
        \nabla \F\left( x;\mathbf{W}\right)  & = \sum_{c = 1}^{C} p\left( c | x\right)  \nabla E\left( x, c ;\mathbf{W}\right)  \\
        & = x\begin{pmatrix}
        p\left( 1| x\right)  \\
        \vdots \\
        p\left( C| x\right) 
    \end{pmatrix}.
    \end{aligned}    
\end{equation}
\end{small}%
where
\begin{small}
\begin{equation}
    p\left(j| x\right)  = \frac{\exp\left( -E\left( x,j\right) \right) }{\exp\left( -\F\left( x\right) \right) }\,, \forall j\in\{1,\cdots,C\}. 
\end{equation}
\end{small}
Lastly, combining Eq.~\eqref{toy_loss}, Eq.~\eqref{gradinet_energy} and Eq.~\eqref{gradient_free_energy}, we have the gradient of $\Ln\left( x, y;\mathbf{W}\right)$ as follows:
\begin{small}
\begin{align}\label{gradient_nllloss}
    \nabla \Ln\left( x, y;\mathbf{W}\right) &= \nabla E\left(x, y;\mathbf{W}\right) - \nabla \F\left(x;\mathbf{W}\right) \nonumber \\
    &= x\begin{pmatrix}
        -p\left( 1| x\right)  \\
        \vdots \\
        1-p\left( y| x\right)  \\
        \vdots \\
        -p\left( C| x\right) 
    \end{pmatrix}.
\end{align}
\end{small}

\subsubsection{Proof of Lemma \ref{lemma_toy_gradient}}
\begin{proof}\label{proof_toy_gradient}
    A correct toy model prediction on the labeled source sample $\left( x, y\right) $ means that the inequality 
    \begin{small}
        \begin{equation}\label{lemma_gridnet_predict_right}
            p\left( y| x\right)  > p\left( j | x\right), \forall j \in \{1,\cdots,C\},j \neq y
        \end{equation}%
    \end{small}
    holds. And from Eq.~\eqref{gradinet_energy}, it is easy to see that 
    \begin{small}
        \begin{equation}
            \left\langle \nabla E\left( x, i;\mathbf{W}\right) , \nabla E\left( x, j;\mathbf{W}\right) \right\rangle  = \left\{
            \begin{aligned}
                & {\| x \|}^2 , & i = j, \\
                & 0, & \text{otherwise},
            \end{aligned}
            \right.
        \end{equation}
    \end{small}
    where $\left\langle \cdot, \cdot \right\rangle $ denotes the inner product. Then with Eq.~\eqref{gradinet_energy}, Eq.~\eqref{gradient_free_energy}, and Eq.~\eqref{gradient_nllloss}, we explicitly calculate the inner product between $\nabla \Ln\left( x, y;\mathbf{W}\right)$ and $\nabla E\left( x, y;\mathbf{W}\right)$
    \begin{small}
        \begin{equation}\label{lemma_gradient_deduction}
            \begin{aligned}
                & \left\langle \nabla \Ln\left( x, y;\mathbf{W}\right) , \nabla \F\left( x;\mathbf{W}\right) \right\rangle  \\
                & = \left\langle \nabla E\left( x, y;\mathbf{W}\right) -\nabla \F\left( x;\mathbf{W}\right) , \nabla \F\left( x;\mathbf{W}\right) \right\rangle  \\
                & = \left\langle \nabla E\left( x, y;\mathbf{W}\right) , \nabla \F\left( x;\mathbf{W}\right) \right\rangle  - \left\langle \nabla \F\left( x;\mathbf{W}\right) , \nabla \F\left( x;\mathbf{W}\right) \right\rangle  \\
                & = \left\langle \nabla E\left( x, y;\mathbf{W}\right) , \sum_{c = 1}^{C} p\left( c | x\right)  \nabla E\left( x, c;\mathbf{W}\right) \right\rangle  \\
                & ~~~~ - \left\langle \sum_{c = 1}^{C} p\left( c | x\right)  \nabla E\left( x, c;\mathbf{W}\right) , \sum_{c = 1}^{C} p\left( c | x\right)  \nabla E\left( x, c;\mathbf{W}\right) \right\rangle  \\
                & = p\left( y | x\right) {\| x \|} ^2  - \sum_{c = 1}^{C} p\left( c | x\right)  ^ 2 {\| x \|} ^2 \\
                & = p\left( y | x\right) \left( 1 - p\left( y | x\right) \right) {\| x \|} ^2 - \sum_{c = 1, c \neq y}^{C} p\left( c | x\right)  ^ 2 {\| x \|} ^2 \\
                & = p\left( y | x\right) \left( \sum_{c = 1, c \neq y}^{C} p\left( c | x\right) \right) {\| x \|} ^2 - \sum_{c = 1, c \neq y}^{C} p\left( c | x\right)  ^ 2 {\| x \|} ^2 \\
                & = \sum_{c = 1, c \neq y}^{C} p\left( c | x\right)  \left( p\left( y | x\right)  - p\left( c | x\right) \right)  {\| x \|} ^2.
            \end{aligned}
        \end{equation}
    \end{small}
    Next with the Eq.~\eqref{lemma_gridnet_predict_right} it is easy to know
    \begin{small}
        \begin{equation}
            \left\langle \nabla \Ln, \nabla \F\right\rangle  > 0.
        \end{equation}
    \end{small}
    It should be pointed out that the sample $x$ with norm $\left\lVert x \right\rVert = 0$ is out of consideration because it is meaningless in our toy example. Consequently, each term of the sum in the last term of Eq.~\eqref{lemma_gradient_deduction} is greater than zero, resulting in the sum of them is greater than zero.

\end{proof}

\subsubsection{Proof of Lemma \ref{lemma_toy_step}}
\begin{proof}\label{proof_toy_step}
    Before embarking on the proof we make some preparation include some conventions of some symbols. In order to show our proof more conveniently, all the symbols with a prime superscript denotes the changed values of the same one without superscript after one step of gradient descent, e.g. $E\left( x,y\right)  = E\left( x,y;\mathbf{W}\right) $ and $E^{\prime}\left( x,y\right)  = E\left( x,y;\mathbf{W^{\prime}}\right) $.

    Firstly, a correct toy model prediction on the labeled source sample $\left( x, y\right) $ indicates that $\forall j \in \{1,\cdots,C\} \backslash \{y\}$
    \begin{small}
        \begin{equation}\label{lemma_step_predict}
            E\left( x, y\right)  < E\left( x, j\right) .
        \end{equation}
    \end{small}
    Then we have a serial of equivalent propositions
    \begin{small}
        \begin{align}\label{equivalent_first}
                \F\left( x, \mathbf{W}\right)  & > \F\left( x, \mathbf{W^{\prime}}\right)  \nonumber\\
                \Leftrightarrow -\log\sum_{c = 1}^{C}\exp\left( -E\left( x, c\right) \right)  & > -\log\sum_{c = 1}^{C}\exp\left( -E^{\prime}\left( x, c\right) \right)  \nonumber\\
                \Leftrightarrow \sum_{c = 1}^{C}\exp\left( -E\left( x, c\right) \right)  & < \sum_{c = 1}^{C}\exp\left( -E^{\prime}\left( x, c\right) \right) . 
            \end{align}    
    \end{small}
    From Eq.~\eqref{toy_energy} and Eq.~\eqref{gradient_nllloss} and one step of gradient descent, we can explicitly express $E^{\prime}$ as follows
    \begin{small}
        \begin{align}
            E^{\prime} \left( x, y\right)  & = \omega_y^{\top} x - \eta \left( 1 - p\left( y| x\right) \right) \left\lVert x \right\rVert ^2 \label{E_prime}\\
            & = E\left( x,y\right)  - \eta\left( 1 - p\left( y| x\right) \right) \left\lVert x \right\rVert ^2,  \nonumber\\
            E^{\prime} \left( x, j\right)  & = \omega_j^{\top} x  + \eta p\left(j| x\right) \left\lVert x \right\rVert ^2 \label{E_original}\\
            & = E\left(x,j\right)  + \eta p\left(j| x\right) \left\lVert x \right\rVert ^2, j\in \{1,\cdots,C\} \backslash \{y\}.  \nonumber
        \end{align}
    \end{small}
    With the Eq.~\eqref{E_prime} and Eq.~\eqref{E_original}, we can continue our deduction of equivalent propositions of Eq.~\eqref{equivalent_first} as follows
    \begin{small}
        \begin{equation}\label{equivalent_second}
            \begin{aligned}
                & \sum_{c = 1}^{C}\exp\left( -E\left( x, c\right) \right)   < \sum_{c = 1}^{C}\exp\left( -E^{\prime}\left( x, c\right) \right)  \\
                & \Leftrightarrow \sum_{c = 1}^{C}\exp\left( -E\left( x, c\right) \right)   < \\
                &\sum_{c = 1, c \neq y}^{C}\exp\left( -E\left( x, c\right)  - \eta p\left( c| x\right) \left\lVert x \right\rVert ^2 \right)  \\
                & + \exp\left( -E\left( x, y\right)  + \eta \left( 1 - p\left( y| x\right) \right)  \left\lVert x \right\rVert ^2 \right)  \\
                & \Leftrightarrow \sum_{c = 1}^{C}\exp\left( E\left( x,y\right) -E\left( x, c\right) \right)  < \\
                &\sum_{c = 1, c \neq y}^{C}\exp\left( E\left( x,y\right) -E\left( x, c\right)  - \eta p\left( c| x\right) \left\lVert x \right\rVert ^2 \right)  \\
                & + \exp\left( \eta\left( 1 - p\left( y| x\right) \right)  \left\lVert x \right\rVert ^2 \right)  \\
                & \Leftrightarrow \sum_{c = 1, c \neq y}^{C}\exp\left( E\left( x,y\right) -E\left( x, c\right) \right)  + e^0 < \\
                &\sum_{c = 1, c \neq y}^{C}\exp\left( E\left( x,y\right) -E\left( x, c\right)  - \eta p\left( c| x\right) \left\lVert x \right\rVert ^2 \right)  \\
                & + \exp\left( \eta\sum_{c=1, c\neq y}^{C}p\left(c| x\right) \left\lVert x \right\rVert ^2 \right) . \\
            \end{aligned}    
        \end{equation}
    \end{small}
    Equipped with Lemma \ref{lemma_exhange_exp}, with $\eta > 0$ and Eq.~\eqref{lemma_step_predict}, we can clearly see that in the final inequality of Eq.~\eqref{equivalent_second}, 
    \begin{small}
    \[
        0 > E\left( x, y\right)  - E\left( x, c\right)  ,  
    \]
    \[
        \eta p\left( c| x\right)  \left\lVert x \right\rVert ^ 2 > 0,
    \]
    \[
        c = 1, \cdots,C, c \neq y, 
    \]
    \end{small}
    which makes it be an example of Lemma \ref{lemma_exhange_exp}, so the original inequality holds for our assumption.
\end{proof}

\begin{lemma}\label{lemma_exhange_exp}
    Let there be $2n + 1$ real numbers $a, a_1, \cdots,a_n$ and $b_1, \cdots,b_n$, if
    \begin{small}
    \[
        a > a_i, b_i > 0, i=1, \cdots,n,
    \]
    we have 
    \[
        e^{a + \sum_{i = 1}^{n} b_i} + \sum_{i = 1}^{n}e^{a_i-b_i} > e^{a} + \sum_{i = 1}^{n}e^{a_i},
    \]
    \end{small}%
\end{lemma}
\begin{proof}\label{proof_exchange_exp}
    Because 
    \begin{small}
    \[
        a > a_i, b_i > 0, i=1, \cdots,n,
    \]
    so
    \[
        e^{a + \sum_{i = j}^n b_i} > e^{a_j}, 1 - e^{-b_j} > 0, j = 1, \cdots, n.
    \]
    \end{small}%
    Then we obtain 
    \begin{small}
    \begin{equation}\label{prepair_exchange_exp}
        e^{a + \sum_{i = j}^n b_i} \left( 1 - e^{-b_j}\right)  > e^{a_j} \left( 1 - e^{-b_j}\right) , j = 1,\cdots,n.
    \end{equation}
\end{small}
    With some simple algebra, from Eq.~\eqref{prepair_exchange_exp} we can find that
    \begin{small}
    \begin{equation}\label{all_exchange}
        e^{a + \sum_{i = j}^n b_i} + e^{a_j - b_j} > e^{a + \sum_{i = j + 1}^n b_i} + e^{a_j} ,j = 1, \cdots,n.
    \end{equation}
    \end{small}
    Equipped with Eq.~\eqref{all_exchange}, we can start our deduction as follows:
    \begin{small}
    \begin{equation}
        \begin{aligned}
            &e^{a + \sum_{i = 1}^{n} b_i} + \sum_{i = 1}^{n}e^{a_i-b_i}  \\
            &= e^{a + \sum_{i = 1}^{n} b_i} + e^{a_1-b_1} + \sum_{i = 2}^{n}e^{a_i-b_i} \\
            &> e^{a + \sum_{i = 2}^{n} b_i} + \sum_{i = 1}^{1} e^{a_i} + \sum_{i = 2}^{n}e^{a_i-b_i} \\
            &= e^{a + \sum_{i = 2}^{n} b_i} + e^{a_2-b_2} + \sum_{i = 1}^{1} e^{a_i} + \sum_{i = 3}^{n}e^{a_i-b_i} \\\
            &>e^{a + \sum_{i = 3}^{n} b_i} + \sum_{i = 1}^{2} e^{a_i} + \sum_{i = 3}^{n}e^{a_i-b_i} \\
			& = e^{a + \sum_{i = 3}^{n} b_i} + e^{a_3-b_3} + \sum_{i = 1}^2 e^{a_i} + \sum_{i = 4}^{n}e^{a_i-b_i} \\
            & > e^{a + \sum_{i = 4}^{n} b_i} + \sum_{i = 1}^{3} e^{a_i} + \sum_{i = 4}^{n}e^{a_i-b_i} \\
            & \qquad \qquad \qquad\vdots \\
            & > e^{a + \sum_{i = j}^{n} b_i} + \sum_{i = 1}^{j - 1} e^{a_i} + \sum_{i = j}^{n}e^{a_i-b_i} \\
            & \qquad \qquad \qquad\vdots \\
            & > e^{a} + \sum_{i = 1}^{n} e^{a_i}.
        \end{aligned}
    \end{equation}
\end{small}%
\end{proof}

\subsubsection{Proof of Theorem \ref{gradient_theorem}}
\paragraph{}In this part, we provide the detailed proof of Theorem~\ref{gradient_theorem}. Before embarking on the proof we make some preparation, here, we define $\beta$-smooth~\cite{convex_optimization}: 
    
\begin{definition}
    A function $f$ is $\beta$-smooth, if the gradient $\nabla f$ is $\beta$-Lipschitz~\cite{convex_optimization} that is 
    \begin{equation}
        \left\lVert \nabla f\left( x\right)  - \nabla f\left( y\right)  \right\rVert \le \beta \left\lVert x - y \right\rVert .
    \end{equation}
\end{definition}

\begin{proof}\label{proof_theorem}
    For convenient, we abbreviate $\Ln\left( x, y;\theta\right) $ to $\Ln\left( \theta\right) $ and $\F\left( x;\theta\right) $ to $\F\left( \theta\right) $. Then with the smooth assumption and Cauchy-Schwarz inequality, we have 
    \begin{small}
        \begin{equation}\label{theoem_smooth}
            \begin{aligned}
                & \F\left( \theta - \eta \nabla \Ln\left( \theta\right) \right)  - \F\left( \theta\right)  - \left\langle  \nabla \F\left( \theta\right) , - \eta \nabla \Ln\left( \theta\right)\right\rangle    \\
                & = \int_{0}^{1} \left\langle \nabla\F\left( \theta - t \eta \nabla \Ln\left( \theta\right) \right) ,- \eta \nabla \Ln\left( \theta\right)\right\rangle   \mathrm{d}t  \\
                & - \left\langle \nabla \F\left( \theta\right) , - \eta \nabla \Ln\left( \theta\right) \right\rangle  \\
                & = \int_{0}^{1} \left\langle \nabla\F\left( \theta - t \eta \nabla \Ln\left( \theta\right) \right)  - \nabla\F\left( \theta\right) ,- \eta \nabla \Ln\left( \theta\right) \right\rangle  \mathrm{d}t  \\
                & \le \left\lvert \int_{0}^{1} \left\langle \nabla\F\left( \theta - t \eta \nabla \Ln\left( \theta\right) \right)  - \nabla\F\left( \theta\right) ,- \eta \nabla \Ln\left( \theta\right)\right\rangle   \mathrm{d}t  \right\rvert \\
                & \le \int_{0}^{1} \left\lVert \nabla\F\left( \theta - t \eta \nabla \Ln\left( \theta\right)  - \nabla\F\left( \theta\right) \right) \right\rVert   \left\lVert - \eta \nabla \Ln\left( \theta\right) \right\rVert \mathrm{d}t  \\
                & \le \int_{0}^{1} t\beta\eta^2 \left\lVert\nabla \Ln\left( \theta\right) \right\rVert ^2  \mathrm{d}t  \\
                & = \frac{\beta\eta^2}{2} \left\lVert\nabla \Ln\left( \theta\right) \right\rVert ^2 .
            \end{aligned}
        \end{equation}
    \end{small}
    From Eq.~\eqref{theoem_smooth}, it is easy to see that 
    \begin{small}
        \begin{equation}\label{inequality_no_assumption}
            \begin{aligned}
                & \F\left( \theta - \eta \nabla \Ln\left( \theta\right) \right)  - \F\left( \theta\right)  \\
                & \le \frac{\beta\eta^2}{2} \left\lVert\nabla \Ln\left( \theta\right) \right\rVert ^2 - \eta \left\langle \nabla \F\left( \theta\right) ,  \nabla \Ln\left( \theta\right)\right\rangle  .
            \end{aligned}
        \end{equation}
    \end{small}
    And by our assumption on the gradient norm and inner product, the right part of Eq.~\eqref{inequality_no_assumption} becomes
    \begin{small}
        \begin{equation}\label{inequality_assumption}
            \begin{aligned}
                &\frac{\beta\eta^2}{2} \left\lVert\nabla \Ln\left( \theta\right) \right\rVert ^2 - \eta \left\langle \nabla \F\left( \theta\right) ,  \nabla \Ln\left( \theta\right) \right\rangle  \\
                & < \frac{\beta G^2\eta^2}{2} - \eta \left\langle \nabla \F\left( \theta\right) ,  \nabla \Ln\left( \theta\right)\right\rangle  \\
                & < \frac{\beta G^2\eta^2}{2} - \eta \varepsilon .
            \end{aligned}
        \end{equation}
    \end{small}
    Let $g\left( \eta\right)  = \frac{\beta G^2\eta^2}{2} - \eta \varepsilon$. Because $\frac{\beta G^2}{2} > 0$, so $g\left( \eta\right) $ is convex for $\forall \eta \in \mathbb{R}$. Then by convexity of $g$, for our fixed range of learning rate $\eta \in \left( 0, \frac{2\varepsilon}{\beta G^2}\right)$,
    \begin{small}
        \begin{equation}\label{g_convex}
            \begin{aligned}
                g\left( \eta\right)  & = g\left( 0 \left( 1 - \frac{\eta}{\frac{2\varepsilon}{\beta G^2}}\right)  + \frac{\eta}{\frac{2\varepsilon}{\beta G^2}} \frac{2\varepsilon}{\beta G^2}\right)  \\
                & < \left( 1 - \frac{\eta}{\frac{2\varepsilon}{\beta G^2}}\right) g\left( 0\right)  + \frac{\eta}{\frac{2\varepsilon}{\beta G^2}}g\left( \frac{2\varepsilon}{\beta G^2}\right)  \\
                & = 0 .
            \end{aligned}
        \end{equation}
    \end{small}
    Eventually, cooperate Eq.~\eqref{inequality_no_assumption}, Eq.~\eqref{inequality_assumption} and Eq.~\eqref{g_convex}, we proof that
    \begin{small}
        \begin{equation}
            \F\left( \theta - \eta \nabla \Ln\left( \theta\right) \right)  - \F\left( \theta\right)  < 0.
        \end{equation}
    \end{small}%
    
\end{proof}

\end{document}